\documentclass{article}

\usepackage{amsmath}
\usepackage{amssymb}
\usepackage{mathtools}
\usepackage{amsthm}
\usepackage{makecell}

\usepackage{microtype}
\usepackage{graphicx}
\usepackage{subcaption}
\usepackage{booktabs} % for professional tables
\usepackage{amsfonts}

\usepackage{newtxmath}
\usepackage{booktabs}
\usepackage{float}
\usepackage{hyperref}
\usepackage{graphicx}
\newcommand{\sfsame}[1]{\scalebox{0.92}{\textsf{#1}}}
\usepackage{enumerate}

\usepackage[accepted]{icml2026}

\usepackage[capitalize,noabbrev]{cleveref}

\theoremstyle{plain}
\newtheorem{theorem}{Theorem}[section]

\theoremstyle{definition}

\newtheorem{assumption}[theorem]{Assumption}
\theoremstyle{remark}
\newtheorem{remark}[theorem]{Remark}

\newcommand{\reals}{{\mbox{\bf R}}}

\renewcommand{\epsilon}{\varepsilon}

\newcommand{\mnorm}[1]{{\vert\kern-0.3ex\vert\kern-0.3ex\vert #1
			\vert\kern-0.3ex\vert\kern-0.3ex\vert}}

\newcommand{\E}{{\mbox{\bf E}}}

\newcommand{\prob}{{\mathbf{P}}}

\newcommand{\clip}{\mathbf{clip}}

\newcommand{\braket}[1]{\langle #1 \rangle}

\newcommand{\cf}{{\it cf.}}
\newcommand{\eg}{{\it e.g.}}
\newcommand{\ie}{{\it i.e.}}
\newcommand{\etc}{{\it etc}}

{\catcode`p =12 \catcode`t =12 \gdef\eeaa#1pt{#1}}      % Get slantfactor
\def\accentadjtext#1{\setbox0\hbox{$#1$}\kern   % Convert it with height
	\expandafter\eeaa\the\fontdimen1\textfont1 \ht0 }
\def\accentadjscript#1{\setbox0\hbox{$#1$}\kern % Convert it with height
	\expandafter\eeaa\the\fontdimen1\scriptfont1 \ht0 }
\def\accentadjscriptscript#1{\setbox0\hbox{$#1$}\kern   % Convert it with height
	\expandafter\eeaa\the\fontdimen1\scriptscriptfont1 \ht0 }
\def\accentadjtextback#1{\setbox0\hbox{$#1$}\kern       % Convert it with height
	-\expandafter\eeaa\the\fontdimen1\textfont1 \ht0 }
\def\accentadjscriptback#1{\setbox0\hbox{$#1$}\kern     % Convert it with height
	-\expandafter\eeaa\the\fontdimen1\scriptfont1 \ht0 }
\def\accentadjscriptscriptback#1{\setbox0\hbox{$#1$}\kern % Convert it with height
	-\expandafter\eeaa\the\fontdimen1\scriptscriptfont1 \ht0 }

\makeatletter
\def\mdots@{\mathinner.\nonscript\!.%
	\ifx\next,.\else\ifx\next;.\else\ifx\next..\else
				\nonscript\!\mathinner.\fi\fi\fi}
\let\ldots\mdots@
\let\cdots\mdots@
\let\dotso\mdots@
\let\dotsb\mdots@
\let\dotsm\mdots@
\let\dotsc\mdots@
\def\vdots{\vbox{\baselineskip2.8\p@ \lineskiplimit\z@
		\kern6\p@\hbox{.}\hbox{.}\hbox{.}\kern3\p@}}
\def\ddots{\mathinner{\mkern1mu\raise8.6\p@\vbox{\kern7\p@\hbox{.}}%
		\raise5.8\p@\hbox{.}\raise3\p@\hbox{.}\mkern1mu}}
\makeatother

\RequirePackage{amsthm}
\newtheoremstyle{descriptive}%
{\topsep}   %{\medskipamount}          % Space above
{\topsep}   %  {\medskipamount}          % Space below
{\rmfamily} % Body font
{}          % Indent
{\bfseries} % Head font
{.}         % Punctuation after thm head
{ }         % Space after thm head
{}          % Thm head spec(?)
\newtheoremstyle{propositional}%
{\topsep}   %  {\medskipamount}          % Space above
{\topsep}   %  {\medskipamount}          % Space below
{\itshape}  % Body font
{}          % Indent
{\bfseries} % Head font
{.}         % Punctuation after thm head
{ }         % Space after thm head
{}          % Thm head spec(?)
\theoremstyle{propositional}
\newtheorem{thm}{Theorem}[section]

\newtheorem{lem}[thm]{Lemma}

\theoremstyle{descriptive}
\newtheorem{deff}[thm]{Definition}
\usepackage[textsize=tiny]{todonotes}
\icmltitlerunning{Clipping Makes Asynchronous SGD Robust to Stragglers}

\begin{document}

\twocolumn[
	\icmltitle{Clipping Makes Distributed and Federated Asynchronous\\ SGD Robust to Stragglers}



	\icmlsetsymbol{equal}{*}

	\begin{icmlauthorlist}
		\icmlauthor{Samuel Erickson}{yyy}
		\icmlauthor{Mikael Johansson}{yyy}
	\end{icmlauthorlist}

	\icmlaffiliation{yyy}{School of EECS, KTH Royal Institute of Technology, Stockholm, Sweden}

	\icmlcorrespondingauthor{Samuel Erickson}{samuelea@kth.se}

	\icmlkeywords{Machine Learning, ICML}

	\vskip 0.3in
]



\printAffiliationsAndNotice{}  % no special notice (required even if empty)

\begin{abstract}
    In modern machine learning, parallelization of training is an important
    strategy for increasing scale. Asynchronous stochastic gradient descent
    (ASGD), which maximizes the utilization of available hardware by avoiding
    waiting for slow workers. However, with constant step sizes, the
    convergence of ASGD is nonetheless affected negatively by slow workers due
    to large delays in updates. At the same time, it has been empirically
    observed in asynchronous training of deep learning models that gradient
    clipping ``stabilizes'' training. In this work, we provide a theoretical
    justification for this behavior, as we show that clipping removes the
    dependence of the maximum delay in the oracle complexity. We employ a
    sub-Weibull model of gradient noise which generalizes sub-Gaussian and
    sub-exponential distributions to more heavy-tailed distributions, motivated
    by empirical observations in deep learning. We show convergence in
    expectation, and the first time in asynchronous optimization, convergence
    with high probability.
\end{abstract}

\section{Introduction}

In recent years, parallelism has become the primary strategy to accommodate the
increasing scale of machine learning models and the datasets used to train
them. The classical parallel distributed optimization algorithm for general
machine learning is Minibatch SGD. However, synchronization in Minibatch SGD
means model updates cannot proceed until all workers have finished, leading to
faster workers idling. This means that iterations are only as fast as the
\emph{slowest} worker. If computation times are uniform across workers, this is
not a problem, but in many real-world settings this is far from being the case.
For instance, these straggler workers are present when there is heterogeneous
hardware or network latency. In some settings, workers may even drop out
entirely due to network issues.

To increase worker utilization and decrease idling, it is natural to consider
removing locking, resulting in asynchronous SGD (ASGD) \citep{Nedic01,
Tsitsiklis03, Feyzmahdavian16, Nguyen18}. Introducing asynchrony can
substantially increase gradient throughput compared to Minibatch SGD, but it
does introduce a new challenge: stale gradients. That is, instead of slowing
down per-iteration run-time, stragglers slow down convergence due to the noise
in their updates that arise from large delays in gradient computations. Both
empirically and theoretically, the convergence of ``vanilla'' ASGD with
constant step size is slowed down by a large maximum delay \citep{Koloskova22,
Wu22}. For this reason, several works have proposed variants of ASGD where the
contributions from workers with large delays are de-emphasized, to obtain
convergence rates that are independent of the maximum delay. However, it was
proved by \citet{Mishchenko22} and \citet{Koloskova22} that vanilla ASGD is in
fact independent of the maximum delay when the gradients are globally bounded.
This naturally raises the following question: \begin{center} \it Does gradient
clipping remove the effect of the maximum delay, making Clipped ASGD robust to
stragglers? \end{center} If clipping can be shown to provide robustness against
the effect of stragglers, it may explain some previously unexplained empirical
behavior. \citet{Chen17} found that in training large deep learning models,
gradient clipping was necessary to ``stabilize'' asynchronous training, but was
not necessary for synchronous training. For heterogeneous optimization (\eg,
federated learning) this may also be particularly significant. Since
delay-adaptive strategies are biased toward faster workers, they do not
converge to a stationary point in the heterogeneous data case. Can we use
clipping to make training robust to stragglers without introducing a bias?

For these reasons, the focus of this work is to provide convergence guarantees
for Clipped ASGD that are robust to stragglers. Moreover, in the asynchronous
setting, high probability guarantees are especially interesting, as guarantees
in expectation do not capture the behavior of a single or few runs. Doing
several runs is of course antithetical to the point of introducing asynchrony,
which is to more effectively utilize available hardware. Despite this, there
are to the best of our knowledge no prior results showing high probability
convergence under asynchrony. Therefore, we focus on providing guarantees with
high probability in addition to guarantees in expectation.

\paragraph{Contributions.} The main contributions of this work are two-fold
and are as follows:
\begin{itemize}
    \item We show that Clipped ASGD is robust to stragglers under a gradient
        noise model which allows us to model the heavy tails seen in deep
        learning. Specifically, we obtain rates that do not depend on the
        maximum delay in both the \emph{homogeneous} and \emph{heterogeneous}
        cases. For the heterogeneous case, this is to the best of our knowledge
        the first asynchronous optimization algorithm that is independent of
        the maximum delay, suggesting that clipping is beneficial in
        asynchronous federated learning where severe stragglers are often
        present.

    \item We show that Clipped ASGD converges with high probability with a
        polylogarithmic dependence on the failure probability, where the degree
        depends on the tail parameter of the gradient noise. This is to the
        best of our knowledge the first result showing high probability
        convergence of an asynchronous optimization algorithm.

\end{itemize}

% In addition to these main theoretical contributions, we conduct experiments on
% asynchronous training of machine learning models. Comparing with Vanilla and
% Delay-adaptive ASGD, these experiments show favorable performance
% of Clipped ASGD.

\section{Related work}

\paragraph{Asynchronous optimization.} Asynchronous optimization for machine
learning has gained much interest since the works of \citet{Agarwal11},
\citet{Recht11} and \citet{Dean12} due to the increasing size of models and the
datasets they are trained on. However, asynchrony in optimization is not a new
idea \citep[\cf][]{Bertsekas89}. Asynchronous gradient descent and SGD date
back to at least the works of \citet{Nedic01} and \citet{Tsitsiklis03},
respectively.

\citet{Feyzmahdavian16} proved an oracle complexity of $O(\sigma^2 / \epsilon^2
+ \tau_\text{max}^2 / \epsilon)$ for smooth convex objectives, where
$\tau_\text{max}$ is the maximum delay. \citet{Mania17} later introduced perturbed
iterate analysis for studying the convergence of asynchronous stochastic
optimization, which \citet{Stich20} adopted in order to improve the complexity.
They showed a $O(\sigma^2 / \epsilon^2 + \tau_\text{max} / \epsilon)$ complexity
for smooth convex objectives, and a $O(\sigma^2 / \epsilon^4 + \tau_\text{max} /
\epsilon^2)$ complexity for smooth non-convex objectives. Similarly,
\citet{Koloskova22} used perturbed iterate analysis but on another virtual
sequence to further improve the complexity. They show that ASGD with constant
step size achieves $O(\sigma^2 / \epsilon^4 + \sqrt{\tau_\text{max} \tau_C} /
\epsilon^2)$ complexity in the smooth non-convex case, where $\tau_C$ is the
number of active workers, \ie, the concurrency.

To mitigate the adverse effects of stragglers, several works have explored
delay-adaptive optimization algorithms \citep{McMahan14, Sra16, Zhang16,
Zheng17, Hannah18, Cohen21, Mishchenko22, Koloskova22, Wu22}. Picky SGD,
proposed by \citet{Cohen21}, was the first ASGD variant to achieve the
complexity $O(\sigma^2 / \epsilon^4 + \tau_C / \epsilon^2)$, completely
removing the dependence on the maximum delay. They achieved this by discarding
gradients that are excessively stale. \citet{Mishchenko22, Koloskova22, Wu22}
later showed that it is also possible to achieve this by adapting the step size
online using the realized delays. Extending the theory to incorporate time
complexity, \cite{Maranjyan25} showed that by dropping gradients that exceed a
certain delay, it is possible to achieve the optimal time complexity.

Another recent development is asynchronous federated learning (FL). Due to the
heterogeneity in hardware capabilities, network latencies, and size of local
datasets among workers, severe stragglers are common in cross-device FL, making
asynchrony very attractive. It has been shown in large production FL
deployments that asynchronous training can lead to substantial speed-ups and
reductions in communication overhead \citep{Huba22}. An early asynchronous FL
method is FedAsync \citep{Xie20}, which was shown to outperform synchronous FL
when the maximum staleness was small. \citet{Koloskova22} also studied a
federated variant of ASGD, showing a similar oracle complexity as vanilla ASGD
in the homogeneous case. In order to make asynchronous training compatible with
privacy mechanisms, \citet{Nguyen22} propose FedBuff which uses buffered
aggregation. \citet{Wang23} take a control variate approach with their method
CA$^2$FL to reduce the effect of data heterogeneity. Recently,
\cite{Maranjyan26} proposed Ringleader ASGD, which achieves the optimal time
complexity under data heterogeneity.

In Table~\ref{tab:complexity}, we summarize the oracle complexities of related
works and ours. For current surveys on asynchronous and parallel optimization,
see \citet{Ben19} and \citet{Feyzmahdavian23}.

\begin{table*}[ht]
	\centering
	\caption{
		Oracle complexities (up to polylogarithmic factors) for homogeneous and
        heterogeneous smooth non-convex optimization. Here $\tau_C$ is the
        concurrency and $\tau_\text{max}$ is the maximum delay, which is always
        larger than $\tau_C$.
	}
    \label{tab:complexity}
	\begin{tabular}{ccc}
		\toprule
		Algorithm & Homogeneous & Heterogeneous \\
		\midrule
		% Minibatch SGD &  $\frac{\tau_\text{max}}{n} \left( \frac{\sigma^2}{\epsilon^4} + \frac{n}{\epsilon^2} \right)$ \\
		\makecell{Vanilla ASGD                  \\ \citep{Koloskova22}}  & $\frac{\sigma^2}{\epsilon^4} + \frac{\sqrt{\tau_\text{max} \tau_C}}{\epsilon^2}$ & $\frac{\sigma^2 + \zeta^2}{\epsilon^4} + \frac{\zeta \tau_C}{\epsilon^3} + \frac{\sqrt{\tau_\text{max} \tau_C}}{\epsilon^2}$ \\
		\midrule
		\makecell{Delay-adaptive ASGD           \\ (\citealt{Cohen21};\\ \citealt{Koloskova22};\\ \citealt{Mishchenko22})} & $\frac{\sigma^2}{\epsilon^4} + \frac{\tau_C}{\epsilon^2}$ & N/A \\
		% Rennala SGD & $\min\limits_{m = 1, \dots, n} \left( \sum\limits_{i=1}^{m} \frac{1}{\tau_i} \right)^{-1} \left( \frac{\sigma^2}{\epsilon^4} + \frac{m}{\epsilon^2} \right)$ \\
		% Lower bound & $\min\limits_{m = 1, \dots, n} \left( \sum\limits_{i=1}^{m} \frac{1}{\tau_i} \right)^{-1} \left( \frac{\sigma^2}{\epsilon^4} + \frac{m}{\epsilon^2} \right)$ & N/A \\
		\midrule
		\makecell{FedBuff                       \\ (\citealt{Nguyen22};\\ \citealt{Wang23})}       & N/A & $\frac{\sigma^2 + \zeta^2}{\epsilon^4} + \frac{\tau_\text{max} \tau_C}{\epsilon^2}$ \\
		\midrule
		% \makecell{CA$^2$FL                      \\ \citep{Wang23}} & N/A & $\frac{\sigma^2}{\epsilon^4} + \frac{\tau_\text{max}^2 + \sigma^2 (\tau_\text{max} + s_\text{max})}{\epsilon^2}$ \\
		% \midrule
		\makecell{Clipped ASGD                  \\ (This work)}  & $\frac{\sigma^2}{\epsilon^4} + \frac{\sigma \tau_C}{\epsilon^3} + \frac{\tau_C}{\epsilon^2}$ & $\frac{\sigma^2 + \zeta^2}{\epsilon^4} + \frac{(\sigma + \zeta) \tau_C}{\epsilon^3} + \frac{\tau_C}{\epsilon^2} $ \\
		\bottomrule
	\end{tabular}
\end{table*}

\paragraph{Gradient clipping.} Gradient clipping is a widely used heuristic for
dealing with instability in optimization, with \citet{Alber98} being an early
work making use of this technique. \citet{Mai21} showed that gradient clipping
can significantly improve the stability of SGD for non-smooth convex functions
with rapidly growing sub-gradients. For non-convex learning, \citet{Mikolov12}
and \citet{Pascanu13} proposed gradient clipping for dealing with the so-called
exploding gradient problem. Later, \citet{Zhang20a} provided a theoretical
justification for using clipping in this context. By generalizing the standard
smoothness assumption to $(L_0, L_1)$-smoothness, they showed that gradient
clipping may enable the use of significantly larger step sizes.
\citet{Zhang20b} then improved upon the dependence on problem-specific
parameters and allowed for momentum. However, \citet{Zhang20b, Zhang20a} both
used the strong uniformly bounded gradient noise model. Relaxing to the
standard bounded variance model, \citet{Koloskova23} showed that Clipped SGD
may still enjoy large step sizes, but has a relatively large \emph{irreducible}
error term.

Another important role that gradient clipping plays is in dealing with
heavy-tailed gradient noise. Several works have showed that the gradient noise
in the training of deep learning models have much heavier tails than
sub-Gaussian distributions model \citep{Simsekli19, Panigrahi19,
Gurbuzbalaban21}. As a response, subsequent works have concerned SGD and its
variants under different models of heavy-tailed noise. 

\citet{Zhang20d} showed that Clipped SGD is convergent in expectation for noise
with bounded $q$-moment for $q \in (1,2]$, whereas vanilla SGD can diverge for
$q < 2$. Later, \citet{Cutkosky21} and \citet{Nguyen23} showed high probability
convergence of Clipped SGD under the same noise model. In addition,
\citet{Cutkosky21} showed that SGD fails to attain a logarithmic dependence on
the failure probability even in the bounded variance setting $q = 2$. More
recently, \citet{Hubler25} connected gradient clipping to normalization and
established improved, parameter-free convergence guarantees for non-convex
optimization under heavy-tailed noise with only bounded $p$-th moments. They
further proved tight sample complexity bounds and high-probability convergence.
Taking another approach to modeling heavy tails, \citet{Li22} and
\citet{Madden24} adopted a sub-Weibull \citep{Vladimirova20} noise assumption,
and showed high probability convergence of Clipped SGD. 

In privacy-preserving machine learning, clipping also plays a structural role
rather than a purely optimization-oriented one. Early work on differentially
private empirical risk minimization and stochastic optimization explicitly
relies on globally bounded gradients to calibrate additive noise mechanisms
\citep{Bassily14, Dwork14}, and this requirement was operationalized in deep
learning through gradient clipping in differentially private SGD
\citep{Abadi16}. \citet{Chen20} later studied the bias that clipping introduces
in differentially private SGD through a geometric lens. Extensions to federated
learning further have emphasized clipping as a prerequisite for aggregating
privatized updates across clients \citep{McMahan18}.

\section{Problem and computational setup}

We consider the unconstrained optimization problem
\begin{equation}
    \label{eq:opt-prob}
	\begin{array}{ll}
		\mbox{minimize} & f(x) = \sum_{i=1}^{n} \E_{\xi \sim \mathcal{D}_i} [F_i(x, \xi)]
	\end{array}
\end{equation}
where $F_i \colon \reals^d \times \mathcal{S} \to \reals$ for a sample space
$\mathcal{S}$. We define $f_i(x) = \E_{\xi \sim \mathcal{D}_i} [F_i(x, \xi)]$.
We denote the standard inner product $\braket{\cdot,\cdot}$ and the Euclidean
norm $\|\cdot\|$. We make the standard smoothness assumption on the objective
$f$:
\begin{assumption}\label{asmp:smooth}
	The objective function $f$ is $L$-smooth, meaning $\|\nabla f(x) - \nabla
		f(y)\| \leq L \|x - y\|$ for every $x,y \in \reals^d$. Furthermore, $f$ is
	bounded from below by $f^\star > -\infty$.
\end{assumption}
The \emph{clipping operator} $\clip_c\colon \reals^d \to \reals^d$ with
\emph{clipping radius} $c>0$ is defined by 
\[
    \clip_c(x) = \min \left\{ 1, \frac{c}{\|x\|} \right\} x
\]
for $x \in \reals^d \setminus \{0\}$ and $\clip_c(0) = 0$. Note that clipping
$x$ is equivalent to projecting $x$ onto the ball $\{y\colon \|y\| \leq
c\}$. 

In this setup, we assume $n$ workers are available for parallel computation,
with each worker $i$ serving as a stochastic oracle that returns $g_t^i(x_t) =
\clip_c(\nabla F_i(x_t,\xi_t^i))$ when queried with $x_t$, where $\xi_t^i \sim
\mathcal{D}_i$. In a real-world system, these workers may be GPUs in a cluster,
cores in a CPU, mobile devices, \etc.

% Furthermore, we assume a \emph{fixed computation model}, where worker $i$ takes
% at most $\tau_i$ seconds to perform an oracle call. Without loss of generality,
% assume that the workers are ordered by computation speed, \ie, $\tau_1 \leq
% 	\tau_2 \leq \dots \leq \tau_n$. For instance, under this model Minibatch SGD
% reaches $\epsilon$-stationarity $\min_t \E \|\nabla f(x_t)\| \leq \epsilon$
% within
% \[
% 	O \left( \frac{\tau_\text{max}}{n} \left( \frac{\sigma^2}{\epsilon^4} + \frac{n}{\epsilon^2} \right) \right)
% \]
% seconds \citep{Khaled23}, where $\sigma^2$ is the variance of the stochastic
% gradients.
% \renewcommand{\arraystretch}{1.5}

% \subsection{Gradient noise model}

Crucial to most analyses of gradient clipping is controlling the error of the
clipped stochastic gradient when the full gradient is small (relative to the
clipping threshold). This generally involves controlling the tail of the
gradient noise, \ie, bounding the probability $\prob (\|\nabla F(x,\xi) -
\nabla f(x)\| > \alpha)$ with a function of $\alpha \geq 0$. The works of
\citet{Zhang20b, Zhang20a} consider the uniformly bounded noise model, when the
tail probability is zero for $\alpha \geq \sigma$ for some parameter $\sigma >
0$. While this technically holds true in typical empirical risk minimization,
it requires a very large value of $\sigma$ in comparison to the ordinary
bounded variance model. \citet{Zhang20a} also note that the noise model can be
relaxed to sub-Gaussian noise, however, this assumption is still strong. Recent
works in deep learning have showed that the stochastic gradients typically
exhibit more heavy-tailed noise than sub-Gaussian random variables models
\citep{Simsekli19, Panigrahi19, Gurbuzbalaban21}. \citet{Koloskova23} consider
the ordinary bounded variance model and use Markov's inequality to bound the
tail. However, Markov's inequality gives a pessimistic bound, leading to a
relatively large error term even for large choices of clipping radius.
\citeauthor{Koloskova23} show that this error is irreducible in worst-case
analysis. Additionally, the convergence results in these works are in
expectation, which do not characterize the behavior of Clipped SGD in a single
run. This motivates the need for a model of gradient noise that gives us better
tools to bound the tails, but still allows us to model the heavy-tails seen in
deep learning.

\begin{deff}\label{def:sub-weibull}
	A random variable $X\colon \mathcal{S} \to \reals$ is called
	\emph{sub-Weibull} if there exists positive quantities $\sigma$ and
	$\theta$ such that
	\[
		\E [\exp((|X| / \sigma)^\frac{1}{\theta})] \leq 2.
	\]
	We denote $X \sim \sfsame{subW}(\theta, \sigma)$.
\end{deff}

The class of sub-Weibull random variables is proposed by \citet{Kuchibhotla22}
and \citet{Vladimirova20}. \citeauthor{Kuchibhotla22} analyze linear regression
and covariance estimation under this model, while \citeauthor{Vladimirova19}
analyze the induced prior distributions in Bayesian neural networks. Note that
sub-Weibull random variables generalize sub-Gaussian and sub-exponential random
variables, which are recovered with $\theta=1/2$ and $\theta=1$, respectively.

\begin{assumption}\label{asmp:sub-weibull}
	The stochastic gradients are unbiased and have sub-Weibull noise, meaning
	$\E_{\xi \sim \mathcal{D}_i} [\nabla F_i(x,\xi)] = \nabla f_i(x)$ and
	\[
		\|\nabla F_i(x,\xi) - \nabla f_i(x)\| \sim \sfsame{subW}(\theta, \sigma)
	\]
	for every $x \in \reals^d$.
\end{assumption}

\begin{figure}[ht!]
    \centering
    \begin{subfigure}[b]{0.48\linewidth}
        \centering
        \includegraphics[width=\textwidth]{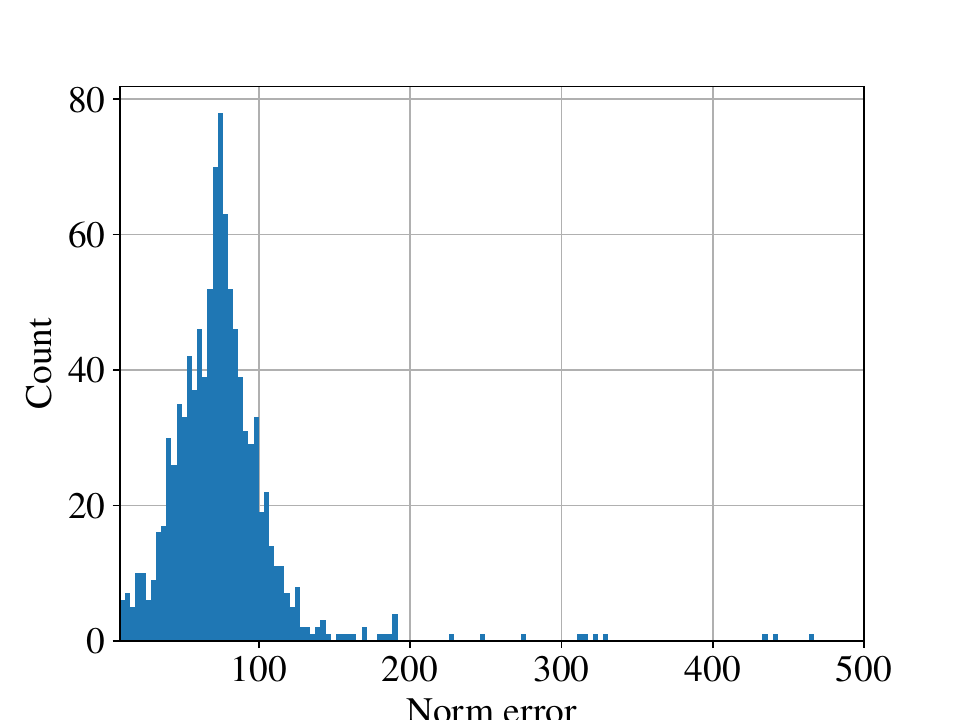}
        \caption{Real distribution}
        \label{fig:real-dist}
    \end{subfigure}
    \begin{subfigure}[b]{0.48\linewidth}
        \centering
        \includegraphics[width=\textwidth]{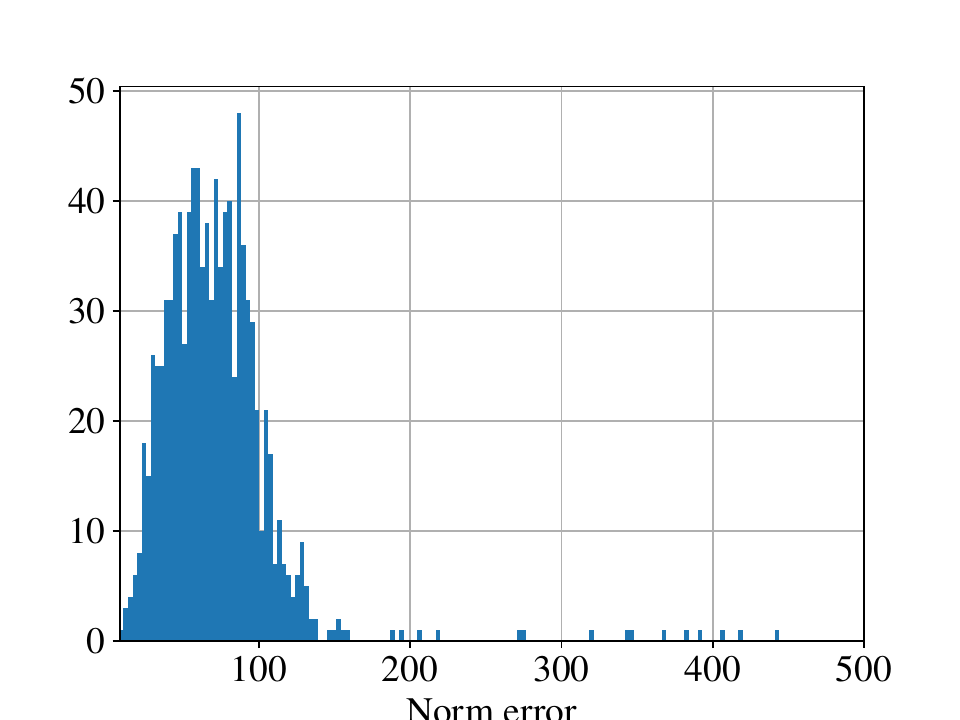}
        \caption{Sub-Weibull $\theta=2.71$.}
        \label{fig:sim-dist}
    \end{subfigure}
    \caption{
        Histograms of (a) gradient errors in training of a ResNet-18 model on
        CIFAR-10 and (b) simulated sub-Weibull distribution. The empirical
        estimate of the tail parameter is $\theta = 2.71$.
    }
    \label{fig:tail-analysis}
\end{figure}

This model of the gradient noise is used by \citet{Li22} and \cite{Madden24},
who provide high probability guarantees for serial SGD. Note that if a random
variable is sub-Weibull then it also has bounded second-moment. In particular,
if $X \sim \text{subW}(\theta, \sigma)$ then $\E|X|^2 \leq C\sigma^2$, where
$C$ only depends on $\theta$. Therefore, the quantity $\sigma^2$ is directly
comparable to the one in the ordinary bounded variance assumption. The quantity
$\theta$ on the other hand determines how heavy the tails are. This makes the
class of sub-Weibull distributions ideal for modeling gradient noise in the
training of modern machine learning models. In Figure~\ref{fig:tail-analysis},
we have plotted the empirical norm gradient errors in training of a ResNet-18
model on the CIFAR-10 dataset, as well as a simulated sub-Weibull distribution
using the estimate of the tail parameter $\theta$ obtained from the ResNet-18
training. We used the tail parameter estimate suggested by
\citet{Vladimirova20}.

\section{Homogeneous setting}

\begin{algorithm}
	\caption{Clipped ASGD (homogeneous setting)}\label{algo:asgd_hom}
	\begin{algorithmic}
        \STATE {\bfseries Input:} Initialization $x_0$, concurrency $\tau_C$,
        step size $\eta > 0$, clipping radius $c > 0$.

        \STATE A subset $\mathcal{C}_0$ of $\tau_C$ workers receive $x_0$ and
        start computing gradients

		\FOR{$t = 0, \dots, T-1$}

        \STATE Worker $i_t$ finishes computing $g_{t - \tau_t}^{i_t}(x_{t - \tau_t})$

        \STATE Server updates $x_{t+1} = x_t - \eta g_{t - \tau_t}^{i_t}(x_{t - \tau_t})$

        \STATE Server selects an inactive worker $j_t \in [n] \setminus
        \mathcal{C}_t$ and updates the active set $\mathcal{C}_{t+1} =
        (\mathcal{C}_t \setminus \{i_t\}) \cup \{j_t\}$

        \STATE Worker $j_t$ receives $x_{t+1}$ and starts computing $g_{t+1}^{j_t}(x_{t+1})$

		\ENDFOR
	\end{algorithmic}
\end{algorithm}

We begin with the homogeneous special case of problem~\eqref{eq:opt-prob},
where the objective functions $f_i$ are identical, \ie, we have $F_1 = \dots =
F_n$ and $\mathcal{D}_1 = \dots = \mathcal{D}_n$. This corresponds to the data
centralized setup, where all workers have access to the full dataset. This
setting gives us much freedom in how we utilize the workers, because we do not
have to worry about introducing optimization bias from over-relying on the fast
workers. Therefore, we consider Algorithm~\ref{algo:asgd_hom}, where workers
are at all times computing (or communicating) clipped gradients. This is the
standard ASGD algorithm considered in \eg\ \citep{Agarwal11, Feyzmahdavian16,
Stich20, Mishchenko22, Koloskova22}, differing only in the use of gradient
clipping. Following \citeauthor{Koloskova22}, we generalize to concurrency
$\tau_C \leq n$, although in the homogeneous case it is common to run with full
concurrency. For instance, the server may prioritize workers that have returned
gradients quickly in the past. Here, the server may select the next worker
$j_t$ out of the inactive workers in any way, due to the homogeneity. Under
maximum concurrency $\tau_C = n$ and the fixed computation model used by \eg\
\citep{Mishchenko22, Tyurin23}, where each worker $i$ takes $h_i$ seconds to
return a clipped gradient, Algorithm~\ref{algo:asgd_hom} takes at most
$(\sum_{i=1}^{n} \frac{1}{h_i})^{-1}$ seconds per oracle call on average.
Compare this to Minibatch SGD which takes $\frac{1}{n} \max_i h_i$ seconds per
oracle call.

Central to recent analyses of ASGD is perturbed iterate analysis, where a
virtual sequence is used to study the actual sequence $\{x_t\}_{t}$ of iterates
\citep{Mania17, Stich20, Mishchenko22, Koloskova22}. In the state-of-the-art
analyses of ASGD \citep{Mishchenko22, Koloskova22}, it is shown that with a
virtual sequence $\tilde x_t$ which evolves almost like serial SGD, the
differences $\tilde x_t - x_t$ can be expressed as a sum of at most $n$
gradient steps. It is precisely here that, in our theoretical analysis,
clipping turns out to be very beneficial for ASGD. It is in fact for the same
reason the strong assumption of globally bounded gradients is beneficial in
\citep{Mishchenko22, Koloskova22}. In particular, with initialization $x_0 =
\tilde x_0$ and step size $\eta$, we define the virtual sequence by
\begin{equation}
    \label{eq:virtual-seq}
    \begin{aligned}
        \tilde x_1 = x_0 - \eta \textstyle \sum_{i \in \mathcal{C}_0} g_{0}^i(x_0), \quad \tilde x_{t+1} = \tilde x_t - \eta g_{t}^{i_t}(x_t)
    \end{aligned}
\end{equation}
and find that the differences $\tilde x_t - x_t$ can be controlled via the
clipping radius $c$ and the step size $\eta$.

\begin{lem}\label{lem:virtual-seq}
    The sequence $\{x_t\}_t$ generated by Algorithm~\ref{algo:asgd_hom} and the
    sequence $\{\tilde x_t\}_t$ defined by \eqref{eq:virtual-seq} satisfy 
    \[
		\|\tilde x_t - x_t\| \leq \eta c \tau_C
	\]
	for all $t = 0,\dots,T-1$.
\end{lem}

Using this result, we may show that the convergence in expectation of Clipped
ASGD does not depend on the maximum delay.

\begin{thm}\label{thm:sub-weibull-hom}
    Suppose Assumptions~\ref{asmp:smooth} and \ref{asmp:sub-weibull} hold. Then
    there exists a constant step size $\eta$ and clipping radius $c$ such that
    for $\epsilon \in (0,1)$, we have $\frac{1}{T} \sum_{t=0}^{T-1} \E\|\nabla
    f(x_t)\| \leq \epsilon$ within
	\[
		\widetilde{O} \left( \frac{\sigma^2}{\epsilon^4} + \frac{\sigma \tau_C}{\epsilon^3} + \frac{\tau_C}{\epsilon^2} \right)
	\]
	iterations of Algorithm~\ref{algo:asgd_hom}.
\end{thm}

Note that eventhough Clipped ASGD differs only in the middle term compared to
the iteration complexity of Delay-adaptive ASGD, despite the algorithm not
compensating directly for delays. In fact, if either $\tau_C = O(\sigma /
\epsilon)$ or $\sigma / \epsilon = O(1)$, then
Theorem~\ref{thm:sub-weibull-hom} shows that Clipped ASGD achieves the same
rate as up to polylogarithmic factors. Clipped ASGD is also favorable to
Vanilla ASGD under globally bounded gradients $\|\nabla f(x)\| \leq G$ or
$G$-Lipschitz loss, which achieves the rate $O(\sigma^2/\epsilon^4 + \tau_C G /
\epsilon^3 + \tau_C / \epsilon^2)$, as $G$ will often be very large. Moreover,
clipping also simplifies tuning. With Vanilla ASGD, the step size to achieve
the rate $O(\sigma^2/\epsilon^4 + \sqrt{\tau_\text{max} \tau_C} / \epsilon^2)$
critically depends on $\tau_\text{max}$, which is generally unknowable a
priori. 

We can furthermore show convergence in high probability, which to the the best
of our knowledge is the first time in asynchronous optimization. 

\begin{thm}\label{thm:high-prob-hom}
    Suppose Assumptions~\ref{asmp:smooth} and \ref{asmp:sub-weibull} hold. Then
    there exists a constant step size $\eta$ and clipping radius $c$ such that
    for $\epsilon \in (0,1)$ and failure probability $\delta \in (0,1)$, we
    have $\prob \left( \frac{1}{T} \sum_{t=0}^{T-1} \|\nabla f(x_t)\| \leq
    \epsilon \right) \geq 1 - \delta$ within
	\[
		\widetilde{O} \left( \frac{\sigma^2 \log^{2 \theta} (1/\delta)}{\epsilon^4} + \frac{\sigma \tau_C \log^\theta (1/\delta)}{\epsilon^3} + \frac{\tau_C}{\epsilon^2} \right)% O \left( \frac{\sigma^2 \log^{2 \theta}(1/\delta \epsilon)}{\epsilon^4} + \frac{\sigma n \log^{\theta}(1/\delta \epsilon)}{\epsilon^{3}} + \frac{n}{\epsilon^2} \right)
	\]
	iterations of Algorithm~\ref{algo:asgd_hom}.
\end{thm}

In particular, in the high probability analysis, we get an additional
martingale difference term
\[
    Z_t = - \eta \braket{\nabla f(\tilde x_t), g_t^{i_t}(x_t) - \E[ g_t^{i_t}(x_t)]}.
\]
Gradient clipping lets us bound the norm of the arguments in the inner product,
where Lemma~\ref{lem:virtual-seq} is necessary for the first argument.
Subsequently applying Freedman's inequality (Lemma~\ref{lem:freedman-ineq}) we
find that with probability atleast $1 - \delta$,
\[
    \begin{aligned}
        \sum_{t=0}^{T-1} Z_t \lesssim &\sum_{t=0}^{T-1} \left( \eta \|\nabla f(x_t)\|^2 + \eta^3 c^2 \tau_C^2 L^2 \right) \\
                                      &+ \eta c^2 \log \frac{2}{\delta} + \sigma^2 \log^{2 \theta} \frac{2T}{\delta}.
    \end{aligned}
\]

Theorem~\ref{thm:high-prob-hom} shows that the rate differs from
Theorem~\ref{thm:sub-weibull-hom} only in a polylogarithmic dependence on the
failure probability, with the degree being determined by the tail parameter
$\theta$ from Assumption~\ref{asmp:sub-weibull}. 

\begin{remark}
    Since clipping does not affect the computation dynamics of ASGD, ignoring
    the small overhead of clipping the gradient, the same time complexity
    analysis applies. In the fixed computation model considered in
    \cite{Tyurin23, Maranjyan25} we assume workers take at most $s_1 \leq \dots
    \leq s_n$ seconds to compute gradients. Under this model, the time
    complexity of Clipped ASGD with full concurrency $\tau_C = n$ is obtained
    by multiplying the oracle complexity by the harmonic sum of computation
    times $(\sum_{i=1}^{n} \frac{1}{s_i})^{-1}$. If the computation times are
    known beforehand, then the concurrency and set of workers can be chosen to
    yield a time complexity of 
    \[
        \widetilde{O} \left( \min_{\tau_C} \left( \sum_{i=1}^{\tau_C} \frac{1}{s_i} \right)^{-1} \left( \frac{\sigma^2}{\epsilon^4} + \frac{\sigma \tau_C}{\epsilon^3} + \frac{\tau_C}{\epsilon^2} \right)  \right).
    \]
    The same reasoning applies to the high probability analysis, since the
    worst-case computation dynamics in this model are deterministic.
\end{remark}

\section{Heterogeneous setting}
\label{sec:heterogeneous}

We now turn to the general heterogeneous case of problem \eqref{eq:opt-prob},
in which every worker has access to its own data distribution and function
$F_i$. In this setup, we have much less freedom in how we utilize workers,
since over-relying on a few fast workers will bias towards their objectives.
For this reason, standard ASGD does not in general converge in this setting
\citep{Mishchenko22}. Thus, we consider the same scheme for worker selection as
for instance \citet{Koloskova22} and \citet{Wang23} in
Algorithm~\ref{algo:asgd_het}. 

\begin{algorithm}
	\caption{Clipped ASGD (heterogeneous setting)}\label{algo:asgd_het}
	\begin{algorithmic}
		\STATE \textbf{Input:} Initialization $x_0$, concurrency $\tau_C$ step size $\eta > 0$, clipping radius $c > 0$.

        \STATE A subset $\mathcal{C}_0$ of $\tau_C$ workers receive $x_0$ and start computing clipped gradients
		\FOR{$t = 0, \dots, T-1$}

        \STATE Worker $i_t$ finishes computing $g_{t - \tau_t}^{i_t}(x_{t - \tau_t})$

		\STATE Server updates $x_{t+1} = x_t - \eta g_{t - \tau_t}^{i_t}(x_{t - \tau_t})$

        \STATE Worker $j_t \sim \sfsame{Uniform}\{1,\dots,n\}$ receives $x_{t+1}$ and schedules $g_{t+1}^{j_t}(x_{t+1})$

		\ENDFOR
	\end{algorithmic}
\end{algorithm}

After a worker has finished computing a clipped gradient, a new worker is
chosen uniformly at random among all workers, and is sent the updated model.
This way, no bias is created toward faster workers. This does however
potentially slow down the wall-clock time of the average oracle call as
compared to standard ASGD. Here, a worker can be sampled several times before
finishing the first gradient computation, in which case a queue of gradients
builds up on that worker.

We make the bounded first-order heterogeneity assumption on the functions $f_i$
which is commonly found in the federated learning literature.

\begin{assumption}\label{asmp:het}
	The functions $f_i$ have bounded heterogeneity, meaning $\|\nabla f_i(x) -
		\nabla f(x)\|^2 \leq \zeta^2$ for every $x \in \reals^d$.
\end{assumption}

\begin{thm}\label{thm:sub-Weibull-het}
    Suppose Assumptions~\ref{asmp:smooth}, \ref{asmp:sub-weibull} and
    \ref{asmp:het} hold. Then there exists a constant step size $\eta$ and
    clipping radius $c$ such that $\frac{1}{T} \sum_{t=0}^{T-1} \E \|\nabla
    f(x_t)\| \leq \epsilon$ within
	\[
		\widetilde{O} \left( \frac{\sigma^2 + \zeta^2}{\epsilon^4} + \frac{(\sigma + \zeta) \tau_C}{\epsilon^3} + \frac{\tau_C}{\epsilon^2} \right) 
	\]
	iterations of Algorithm~\ref{algo:asgd_het}.
\end{thm}

Surprisingly, Theorem~\ref{thm:sub-Weibull-het} shows that gradient clipping
can remove the dependence on the maximum delay even in the heterogeneous case,
where delay-adaptive schemes which do not in general converge. Compared this
rate with the rate $O((\sigma^2 + \zeta^2) / \epsilon^4 + \zeta \tau_C /
\epsilon^3 + \sqrt{\tau_C\tau_\text{max}} / \epsilon^2)$ heterogeneous version
of Vanilla ASGD, we have a significant improvement when delays are large.
Moreover, for this result \citet{Koloskova22} require an additional assumption
that the average delay of a worker is independent from the number of times that
worker is sampled, which may not be entirely realistic.

\begin{thm}
    Suppose Assumptions~\ref{asmp:smooth}, \ref{asmp:sub-weibull} and
    \ref{asmp:het} hold. Then there exists a constant step size $\eta$ and
    clipping radius $c$ such that for $\epsilon \in (0,1)$ and failure
    probability $\delta \in (0,1)$, we have $\prob \left( \frac{1}{T}
    \sum_{t=0}^{T-1} \|\nabla f(x_t)\| \leq \epsilon\right) \geq 1 - \delta$
    within
	\[
		\widetilde{O} \left( \frac{(\sigma^2 + \zeta^2) \log^{2 \theta} (1/\delta)}{\epsilon^4} + \frac{(\sigma + \zeta)\tau_C\log^{\theta} (1/\delta)}{\epsilon^3} + \frac{\tau_C}{\epsilon^2} \right)
	\]
    iterations of Algorithm~\ref{algo:asgd_het}.
\end{thm}

This high-probability convergence guarantee may be particularly important in
federated learning, as training runs in real-world deployments are typically
done at most a few times because each run incurs substantial logistical,
computational, and organizational cost. Coordinating participation across a
large, dynamic population of client devices requires careful scheduling,
incentives, and system overhead, and client availability cannot be reliably
reproduced across runs. Moreover, federated training pipelines are tightly
coupled to product release cycles, privacy reviews, and regulatory approvals,
making repeated experimentation slow and expensive. As a result, retraining is
not a cheap or repeatable process, making high-probability guarantees
particularly relevant.

\begin{remark}
    Quantifying the time complexity is more delicate in the heterogeneous case
    due to the sampling scheme required to preserve the unbiasedness of the
    stochastic gradients. For this reason, we have included experiments on the
    homogeneous and heterogeneous setting to empirically measure this slowdown.
\end{remark}

\section{Numerical experiments}

We examine the effect of using gradient clipping in asynchronous optimization
of neural network parameters. We conduct several experiments on the CIFAR-10
\citep{Krizhevsky09} and Shakespeare \citep{Karpathy15} datasets, in the
homogeneous (shared memory) setting. To allow for different delay setups, we
simulate asynchronous training with $n = 16$ workers, where half of them take 1
time unit to compute a gradient, and the other half takes $D \in \{4, 8\}$ time
units. We use full concurrency $\tau_C = n$ in all runs. All experiments are
averaged over three random seeds, and $2\sigma$ error bars are shown in the
figures. The code used to run these experiments is available at
\url{https://github.com/samericks/clipped-asgd}.

\subsection{Homogeneous setting}

We compare Clipped ASGD with three baseline methods in the homogeneous setting:
Vanilla ASGD with constant step size, Delay-adaptive ASGD with step size rule
proposed by \citet{Koloskova22}, and Ringleader ASGD \citep{Maranjyan25}. For
Clipped ASGD, we choose the best clipping radius $c\in \{2^k \colon k =
-1,\dots,2\}$, and for Ringleader ASGD we choose the best delay threshold $R
\in \{2^k \colon k = 1,\dots,4\}$. In all cases, the optimal hyperparameter
value lies strictly within the search range.

\paragraph{CIFAR-10.} We train a ResNet-18 model \citep{He16} on CIFAR-10. We
sweep over $\eta \in \{2^{-9}, \dots, 2^{-1}\}$ and measure the simulated
wall-clock time required to reach 80\% test accuracy. We terminate runs that do
not reach the target within 4{,}000 time units.

Figure~\ref{fig:cifar10} shows the results for delay factors $D \in \{4,8\}$.
Clipped ASGD consistently improves over Vanilla ASGD and Delay-adaptive ASGD
across all delay settings. Here, Clipped ASGD reduces the minimum wall-clock
time by $1.8\times$ relative to Vanilla ASGD, and by $1.5\times$ relative to
Delay-adaptive and Ringmaster ASGD. 

As predicted by the theory, increasing the delay primarily affects Vanilla
ASGD, which requires substantially smaller step sizes to remain stable. While
the best wall-clock time is not significantly worsened for Vanilla ASGD, it
requires more fine-grained tuning to converge within the time budget. In
contrast, the clipped and delay-adaptive methods remain robust to larger
delays, with no significant shifts due to larger delays.

\begin{figure}[ht!]
    \centering
    \begin{subfigure}[b]{\linewidth}
        \centering
        \includegraphics[width=\textwidth]{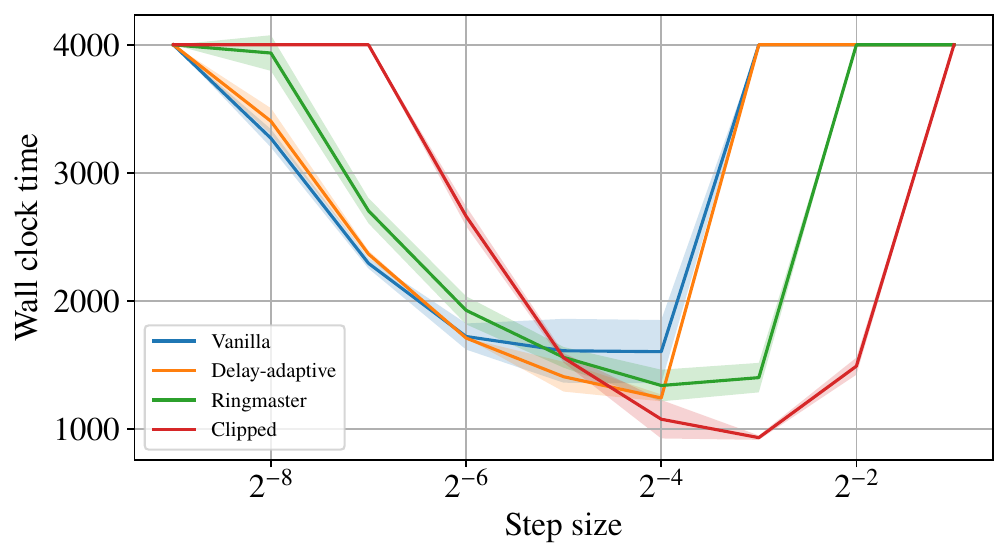}
        \caption{$D=4$}
    \end{subfigure}
    \begin{subfigure}[b]{\linewidth}
        \centering
        \includegraphics[width=\textwidth]{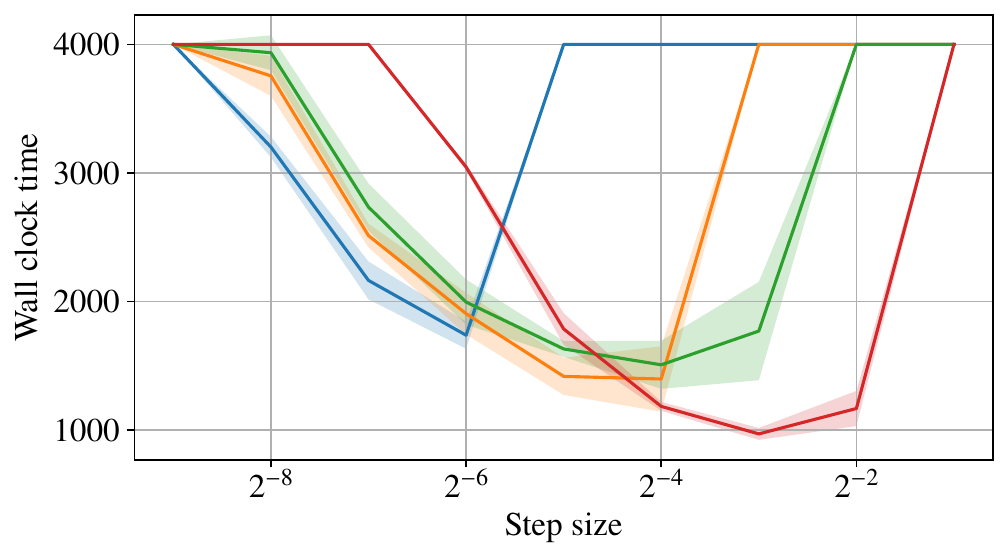}
        \caption{$D=8$}
    \end{subfigure}
    \caption{
        Simulated wall-clock time to reach 80\% test accuracy on CIFAR-10
        dataset with ResNet-18 architecture, when half of the 16 workers are
        $D$ times slower than the other half. The average time per oracle call
        here is $0.108$ and $0.123$ time units for $D=4$ and $D=8$,
        respectively.
    }
    \label{fig:cifar10}
\end{figure}

\paragraph{Shakespeare.} We next consider next-word prediction on the
Shakespeare dataset using an LSTM architecture \citep{Hochreiter97} with
dropout rate 0.2. We sweep over step sizes $\eta \in \{2^{-3}, \dots, 2^3\}$
and measure the simulated wall-clock time required to reach test perplexity
5.0. Runs are terminated after 4{,}000 time units. 

The results are shown in Figure~\ref{fig:shakespeare}. Clipped ASGD again
outperforms both Vanilla and Delay-adaptive ASGD. For $D=4$, Clipped ASGD
achieves the target perplexity $1.8\times$ faster than Vanilla ASGD,
$2.1\times$ faster than Delay-adaptive ASGD, and $1.8\times$ faster than
Ringmaster ASGD. For $D=8$, Clipped ASGD is $2\times$ faster than Vanilla ASGD,
$2.2\times$ faster than Delay-adaptive ASGD, and $1.4\times$ faster than
Ringmaster ASGD.

\begin{figure}[ht!]
    \centering
    \begin{subfigure}[b]{\linewidth}
        \centering
        \includegraphics[width=\textwidth]{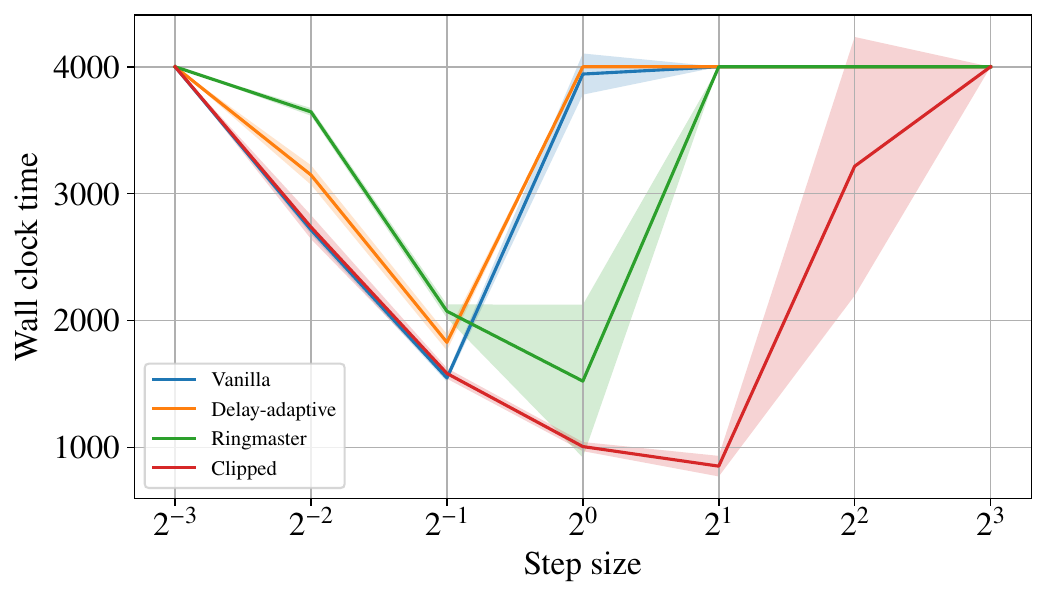}
        \caption{$D=4$}
    \end{subfigure}
    \begin{subfigure}[b]{\linewidth}
        \centering
        \includegraphics[width=\textwidth]{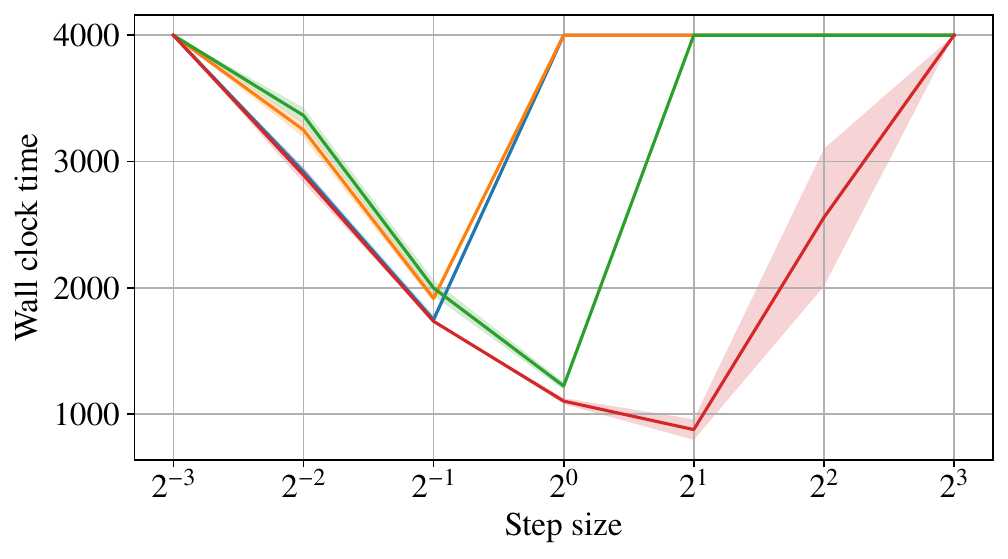}
        \caption{$D=8$}
    \end{subfigure}
   \caption{
        Simulated wall-clock time to reach test perplexity 5.0 on Shakespeare
        dataset with LSTM architecture, when half of the 16 workers are $D$
        times slower than the other half. The average time per oracle call here
        is $0.108$ and $0.123$ time units for $D=4$ and $D=8$, respectively.
    }
    \label{fig:shakespeare}
\end{figure}

\subsection{Heterogeneous setting} 

We now consider asynchronous optimization under data heterogeneity. We compare
Clipped ASGD against Vanilla ASGD and Ringleader ASGD \citep{Maranjyan26}.
Clipped and Vanilla ASGD use the uniform sampling scheme described in
Section~\ref{sec:heterogeneous}. As before, we tune the clipping radius over $c
\in \{2^k : k=-1,\dots,2\}$. 

Note however that in practical FL deployments, exact unbiasedness is often
sacrificed in favor of faster process times. For example, many synchronous FL
deployments aggregate updates once only a fraction (\eg, 80\%) of workers have
responded, rather than waiting for the slowest participants. Thus, asynchrony
can in fact \emph{decrease bias}, since the slowest workers may participate at
all \citep{Huba22}.

\paragraph{Label-skew CIFAR-10.} We revisit CIFAR-10, but now distribute the
data heterogeneously across workers using a Dirichlet partitioning scheme.
Following \citet{Nguyen22, He20, Diao20}, we sample client label distributions
from a Dirichlet distribution with parameter $\alpha = 0.5$. We use a two-layer
CNN architecture and sweep over step sizes $\eta \in \{2^{-9}, \dots,
2^{-1}\}$. 

We measure the simulated wall-clock time required to reach 70\% test accuracy,
with a maximum runtime of 8{,}000 for $D=4$ and 12{,}000 for $D=8$. The results
are shown in Figure~\ref{fig:label_skew_cifar10}. Clipped ASGD consistently
improves over both baselines across heterogeneity levels. With delay factor
$D=4$, clipping reduces the minimum wall-clock time by $1.2\times$ relative to
Vanilla ASGD and Ringleader ASGD. With $D=8$, Clipped ASGD is $1.3\times$
faster than Vanilla ASGD, and $1.2\times$ faster than Ringleader ASGD. Note
that due to the sampling schemes used to avoid oversampling the fast workers,
the maximum delay is also effectively controlled at the expense of process
time, which may explain why clipping yields a smaller improvement here than in
the homogeneous experiments.

\begin{figure}[ht!]
    \centering
    \begin{subfigure}[b]{\linewidth}
        \centering
        \includegraphics[width=\textwidth]{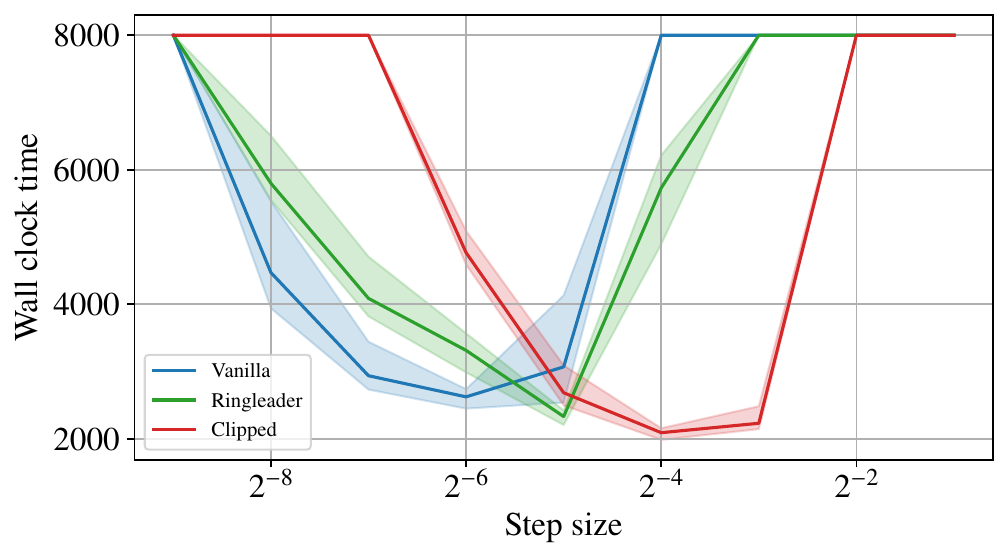}
        \caption{$D=4$}
    \end{subfigure}
    \begin{subfigure}[b]{\linewidth}
        \centering
        \includegraphics[width=\textwidth]{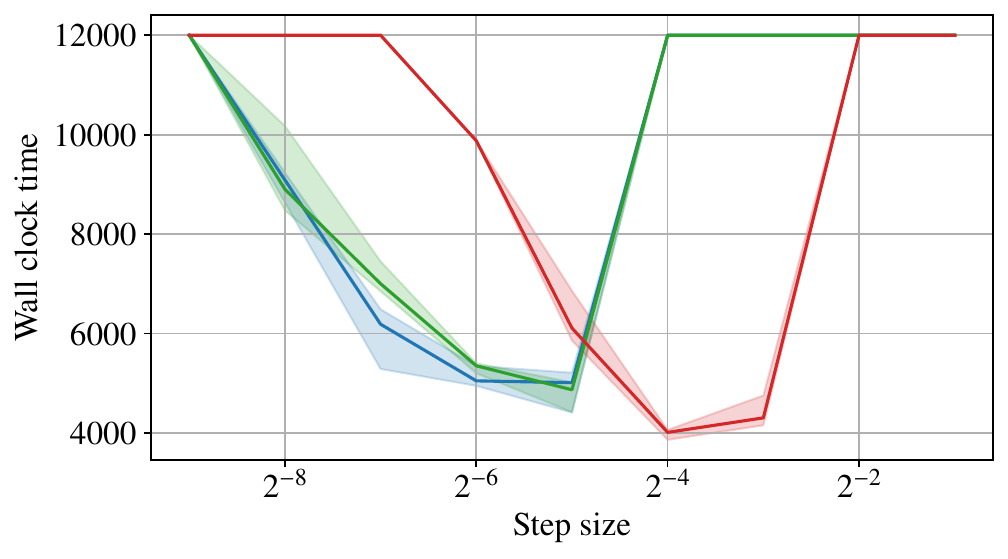}
        \caption{$D=8$}
    \end{subfigure}
   \caption{
        Simulated wall-clock time to reach test metric target on label skew
        CIFAR-10 dataset with a CNN architecture, when half of the 16 workers
        are $D$ times slower than the other half. The average time per oracle
        call is $0.337$ and $0.668$ time units for $D=4$ and $D=8$
        respectively.
    }
    \label{fig:label_skew_cifar10}
\end{figure}

\section{Conclusions and future work}

This paper shows that gradient clipping fundamentally alters the effect of
asynchrony in stochastic optimization. In particular, clipping removes
dependence on the maximum delay in the convergence behavior of ASGD, yielding
methods that are provably robust to stragglers. We provide convergence
guarantees in both expectation and high probability, and show that these
guarantees hold under homogeneous and heterogeneous settings. Empirically, we
demonstrate that clipping consistently improves asynchronous training and can
substantially outperform both Vanilla and Delay-adaptive ASGD across a range of
architectures and delay regimes.

\paragraph{Future work.} Our results suggest a broader connection between norm
control and robustness to asynchrony, raising several directions for future
investigation.

First, it is natural to ask how these effects extend beyond the standard smooth
setting. In particular, under weaker assumptions such as $(L_0,
L_1)$-smoothness, the interaction between gradient norms and staleness may
become more pronounced, potentially making clipping even more beneficial.

Second, recent optimizers such as Muon \citep{Jordan24} and Scion
\citep{Pethick25} achieve strong empirical performance by explicitly
controlling update norms, through orthogonalization or constrained update
steps. Since clipping provides a simple mechanism for norm control in
asynchronous settings, it is natural to ask whether asynchronous variants of
these methods inherit similar robustness to delays. We believe understanding
this interaction between optimizer geometry and asynchrony is a promising
direction for future work.

\section*{Impact statement}

This paper presents work whose goal is to advance the field of Machine
Learning. There are many potential societal consequences of our work, none
which we feel must be specifically highlighted here.

\section*{Acknowledgement}

This work was supported in part by the Knut and Alice Wal- lenberg Foundation
through project KAW 2022.0050 and by the Wallenberg AI, Autonomous Systems and
Software Program (WASP).

\newpage
\appendix
\onecolumn

\section{Preliminaries}

\subsection{Properties of Sub-Weibull random variables}

\begin{thm}[\citealt{Vladimirova20}]\label{thm:sub-weibull}
	Let $X\colon \mathcal{S} \to \reals$ be a random variable. Then the
	following are equivalent:

	\begin{enumerate}[(i)]
		\item There exists a $K_1 > 0$ such that $\prob(|X| \geq x) \leq 2
			      \exp(-(x/K_1)^\frac{1}{\theta})$ for all $x \geq 0$.

		\item There exists a $K_2 > 0$ such that $\E [|X|^q]^{1/q} \leq K_2^{1/q}
			      q^\theta$ for all $q \in (1, \infty)$.

		\item There exists a $K_3 > 0$ such that $\E
			      [\exp((\lambda|X|)^\frac{1}{\theta})] \leq \exp((\lambda
			      K_3))^\frac{1}{\theta})$.

		\item There exists a $K_4 > 0$ such that $\E [\exp((|X| /
				      K_4)^\frac{1}{\theta})] \leq 2$.

	\end{enumerate}

	Additionally, the constants $K_1,\dots,K_4$ differ at most by a factor that
	depends on $\theta$ (in particular, $K_1 = K_4$).

\end{thm}

\begin{lem}\label{lem:sub-weibull}
	If $X \sim \sfsame{subW}(\theta, \sigma)$, then
	\begin{enumerate}[(i)]
        \item $\prob\left( |X| \leq \sigma \log^\theta(2/\delta) \right) \geq 1 - \delta$,

        \item $\E[X^2] \leq 2 \Gamma(2 \theta + 1) \sigma^2$.
	\end{enumerate}
\end{lem}

\subsection{Inequalities and lemmas}

\begin{lem}[Smoothness inequality]\label{lem:smooth}
	If Assumption~\ref{asmp:smooth} holds, then
	\[
		f(x) \leq f(y) + \braket{x - y, \nabla f(y)} + \frac{L}{2} \|x - y\|^2
	\]
	for all $x,y \in \reals^d$.
\end{lem}

\begin{lem}[Young's inequality]
    \label{lem:young}
	For any pair of vectors $x,y \in \reals^d$,
	\[
		\braket{x,y} \leq \frac{1}{2} \|x\|^2 + \frac{1}{2} \|y\|^2.
	\]
\end{lem}

\begin{lem}[Freedman's inequality]\label{lem:freedman-ineq}
    Consider a martingale difference sequence $\{Z_t\}_{t}$ adapted to a
    filtration $\mathcal{G}_t$ and suppose that $|Z_t| \leq \ell$ almost
    surely. Then with probability at least $1 - \delta$ it holds for any
    $\rho \in (0,1)$ that
	\[
        \sum_{t=0}^{T-1} Z_t \leq \frac{\rho}{\ell} \sum_{t=0}^{T-1} \E [Z_t^2 \mid \mathcal{G}_t] + \frac{\ell}{\rho} \log (1/\delta).
	\]
\end{lem}

\section{Proof of Theoretical Results}

For notational brevity, we write $g_t^i = g_t^i(x_t)$, and denote the
expectation and probability conditioned upon the $\sigma$-algebra
$\mathcal{F}_t$ generated by $\{g_s\}_{s=0}^{t-1}$ by $\E_t [\cdot]$ and
$\prob_t \{\cdot\}$, respectively. Note that the sequences $\{x_t\}_t$ and
$\{\tilde x_t\}_t$ are adapted to $\{\mathcal{F}_t\}_t$, and that $f$ and
$\nabla f$ are measurable due to everywhere differentiability of $f$, so $\E_t
[f(\tilde x_t)] = f(\tilde x_t)$ and $\E_t [\nabla f(x_t)] = \nabla f(x_t)$
almost surely. Define the virtual sequence $\{\tilde x_t\}_{t=0}^T$ by
\[
    \tilde x_0 = x_0, \quad \tilde x_1 = x_0 - \eta \sum_{i\in\mathcal{C}_0} g_{0}^i, \quad \tilde x_{t+1} = \tilde x_t - \eta g_t^{i_t},
\]
We first present a key inequality for our analysis, which is essentially a
corollary of \citep[Lemma 1,][]{Mishchenko22}.

\begin{lem}[Virtual iterate bound]\label{lem:virtual-ineq}
	The sequences $\{x_t\}_t$ and $\{\tilde x_t\}_t$ satisfy
	\[
		\|\tilde x_t - x_t\| \leq \eta c \tau_C
	\]
	for all $t = 0,\dots,T-1$.
\end{lem}

\begin{proof}
    % Let $e_t = x_t - \tilde x_t$, and note that
    % \[
    %     e_t = e_{t-1} - \eta g_{t - \tau_t} + \eta g_t = \dots = e_0 - \eta \sum_{s=1}^{t} g_{s - \tau(s)} + \eta \sum_{s=0}^{t-1} g_t
    % \]
    We have that $\tilde x_t = x_0 - \eta \sum_{s=1}^{t-1} g_s - \eta
    \sum_{i=1}^{n} g_0^i$ and $x_t = x_0 - \eta \sum_{s=1}^{t-1} g_{s -
    \tau(s)}^{i(s)}$. But all clipped gradients apart from the $\tau_C$ clipped
    gradients being computed at iteration $t$ are included in the sum
    $\sum_{s=1}^{t-1} g_{s-\tau(s)}^{i(s)}$, hence
	\[
        \|\tilde x_t - x_t\| = \eta \left\| \sum_{s=1}^{t-1} g_s^{i(s)} - \sum_{s=1}^{t-1} g_{s - \tau(s)}^{i(s)} + \sum_{i\in\mathcal{C}_0} g_0^i \right\| \leq \eta c \tau_C,
	\]
	noting that $\|g_s\| \leq c$ for all $s$.
\end{proof}

Central to our analysis is bounding the squared norm error between the clipped
stochastic gradient and the full gradient, when the full gradient is small.

\begin{lem}[Bias and moment bounds]\label{lem:clip-error}
    Suppose Assumptions~\ref{asmp:smooth}, \ref{asmp:sub-weibull} and
    \ref{asmp:het} hold. Then 
    \begin{equation}\label{eq:grad-error}
        \E_t \|\nabla F_{i_t}(x_t, \xi_t^{i_t}) - \nabla f(x_t)\|^2 \leq 4 \sigma^2 \Gamma(2 \theta + 1) + 2 \zeta^2.
    \end{equation}
    Additionally, if $\|\nabla f(x_t)\| < c / 2$ and $c > 2 \zeta$, then
    \begin{equation}\label{eq:clip-error}
        \|\E_t g_t^{i_t} - \clip_c(\nabla f(x_t))\|^2 \leq \left( 8 \sigma^2 \Gamma(2 \theta + 1) + 4 \zeta^2 \right) \exp \left( - \left( \frac{c - 2 \zeta}{4 \sigma} \right)^\frac{1}{\theta} \right),
	\end{equation}
	and
	\begin{equation}\label{eq:squared-clip}
        \E_t \|g_t^{i_t}\|^2 \leq 24 \sigma^2 \Gamma(2 \theta + 1) + 8 \zeta^2 + 2 \|\nabla f(x_t)\|^2.
	\end{equation}
\end{lem}

\begin{proof}
    For \eqref{eq:grad-error}, we first apply Young's inequality 
    \[
        \E_t \|\nabla F_{i_t}(x_t, \xi_t^{i_t}) - \nabla f(x_t)\|^2 \leq 2 \E_t \|\nabla F_{i_t}(x_t, \xi_t^{i_t}) - \nabla f_{i_t}(x_t)\|^2 + 2 \E_t \|\nabla f_{i_t}(x_t) - \nabla f(x_t)\|^2.
    \] 
    Then using (ii) from Lemma~\ref{lem:sub-weibull} on the first term, and
    Assumption~\ref{asmp:het} on the second, we get the stated bound
    \eqref{eq:grad-error}.

    \paragraph{Bounding $\prob_t\{\|\nabla F_{i_t}(x_t, \xi_t^{i_t})\| \geq
    c\}$.} Define the indicator function $\chi_t = \mathbf{1} \{\|\nabla
    F_{i_t}(x_t, \xi_t^{i_t})\| \geq c\}$. We have that
	\[
		\begin{aligned}
            \prob_t \{\chi_t = 1\} & \leq \prob_t \left\{ \|\nabla F_{i_t}(x_t, \xi_t^{i_t}) - \nabla f_{i_t} (x_t)\| + \|\nabla f_{i_t}(x_t) \| > c \right\} \\
                                   & \leq \prob_t \left\{ \|\nabla F_{i_t}(x_t, \xi_t^{i_t}) - \nabla f_{i_t} (x_t)\| + \|\nabla f_{i_t}(x_t) - \nabla f(x_t) \| + \|\nabla f(x_t)\| > c \right\} \\
			                       & \leq \prob_t \left\{ \|\nabla F_{i_t}(x_t, \xi_t^{i_t}) - \nabla f_{i_t} (x_t)\| > \frac{c}{2} - \zeta \right\}.
		\end{aligned}
	\]
    Applying the sub-Weibull concentration inequality (i) in
    Theorem~\ref{thm:sub-weibull}, we get the bound 
    \begin{equation}\label{eq:conc-ineq}
        \prob_t \{\chi_t = 1\} \leq 2 \exp \left( - \left( \frac{c - 2\zeta}{4 \sigma} \right)^\frac{1}{\theta} \right).
    \end{equation}
    This yields the exponential term in \eqref{eq:clip-error}. For the
    polynomial bound, we bound the probability as 
    \[
		\begin{aligned}
            \prob_t \{\chi_t = 1\} & \leq \prob_t \left\{ \|\nabla F_{i_t}(x_t, \xi_t^{i_t}) - \nabla f (x_t)\| + \|\nabla f(x_t) \| > c \right\} \\
                                   & \leq \prob_t \left\{ \|\nabla F_{i_t}(x_t, \xi_t^{i_t}) - \nabla f (x_t)\| > \frac{c}{2} \right\}.
		\end{aligned}
	\]
    and apply Markov's inequality to get
    \begin{equation}\label{eq:markov-conc}
        \prob_t \{\chi_t = 1\} \leq \frac{4 \E_t\|\nabla F_{i_t}(x_t, \xi_t^{i_t}) - \nabla f(x_t)\|^2}{c^2} \leq \frac{16 \Gamma(2 \theta + 1) \sigma^2 + 4 \zeta^2}{c^2}.
    \end{equation}

    \paragraph{Bias bound.} Turning to the quantity $\|\E_t g_t^{i_t} -
    \clip_c(\nabla f(x_t))\|^2$, we first note that $\clip_c(\nabla f(x_t)) =
    \nabla f(x_t) = \E_t [\nabla F_{i_t}(x_t,\xi_t^{i_t})]$ since we have
    assumed $\|\nabla f(x_t)\| < c/2$. Thus, we have that 
	\[
		\begin{aligned}
			\|\E_t g_t^{i_t} - \nabla f(x_t)\|^2 & = \left\| \E_t \left[ \left( 1 - \frac{c}{\|\nabla F_{i_t}(x_t, \xi_t^{i_t})\|} \nabla F_{i_t}(x_t, \xi_t^{i_t}) \right) \chi_t \right] \right\|^2     \\
			                                     & \leq \E_t \left[ \left\| \left( 1 - \frac{c}{\|\nabla F_{i_t}(x_t, \xi_t^{i_t})\|} \right) \nabla F_{i_t}(x_t, \xi_t^{i_t})  \right\| \chi_t \right]^2 \\
                                                 & = \E_t \left[  \left| c - \|\nabla F_{i_t}(x_t, \xi_t^{i_t})\| \right| \chi_t \right]^2 \\
                                                 & \leq \E_t \left[  \left| \|\nabla f(x_t)\| - \|\nabla F_{i_t}(x_t, \xi_t^{i_t})\| \right| \chi_t \right]^2.
		\end{aligned}
	\]
    Now, using the reverse triangle inequality followed by the Cauchy-Schwarz
    inequality,
 	\[
		\begin{aligned}
             \E_t \left[  \left| \|\nabla f(x_t)\| - \|\nabla F_{i_t}(x_t, \xi_t^{i_t})\| \right| \chi_t \right]^2 & \leq \E_t \left[  \|\nabla f(x_t) - \nabla F_{i_t}(x_t, \xi_t^{i_t})\|  \chi_t \right]^2                                                      \\
                                                 & \leq \E_t [\|\nabla f(x_t) - \nabla F_{i_t}(x_t, \xi_t^i)\|^2] \E_t[\chi_t^2]                                                                      \\
			                                     & \leq (4 \sigma^2 \Gamma(2 \theta + 1) + 2 \zeta^2) \E_t[\chi_t^2]                                                                                        \\
			                                     & = (4 \sigma^2 \Gamma(2 \theta + 1) + 2 \zeta^2) \prob_t\{\chi_t = 1\},
		\end{aligned}
	\]
    where we applied \eqref{eq:grad-error} and in the last inequality we used
    (ii) in Lemma~\ref{lem:sub-weibull}. We arrive at the stated result
    \eqref{eq:clip-error} by applying \eqref{eq:conc-ineq}.

    \paragraph{Second moment bound.} Turning to $\E_t \|g_t^{i_t}\|^2$, we have
    that
	\[
		\begin{aligned}
			\E_t \|g_t^{i_t}\|^2 & = \E_t [\chi_t \|g_t^{i_t}\|^2 + (1 - \chi_t) \|g_t^{i_t}\|^2] \\
                           & = \prob_t(\chi_t = 1) \E_t [\|g_t^{i_t}\|^2 \mid \chi_t = 1] + \E_t [(1 - \chi_t) \|g_t^{i_t}\|^2] \\
			               & = \prob_t(\chi_t = 1) c^2 + \E_t \left[ (1 - \chi_t) \|\nabla F_{i_t}(x_t, \xi_t^{i_t})\|^2 \right] \\
			               & \leq \prob_t(\chi_t = 1) c^2 + \E_t \left[ \|\nabla F_{i_t}(x_t, \xi_t^{i_t})\|^2 \right].
		\end{aligned}
	\]
    Subsequently, we bound the term $\E_t \left[ \|\nabla F_{i_t}(x_t, \xi_t^{i_t})\|^2
    \right]$ by applying Young's inequality and (ii) in Lemma~\ref{lem:sub-weibull}:
	\[
		\begin{aligned}
			\E_t \left[ \|\nabla F_{i_t}(x_t, \xi_t^{i_t})\|^2 \right] & = \E_t \left[ \|(\nabla F_{i_t}(x_t, \xi_t^{i_t}) - \nabla f(x_t)) + \nabla f(x_t)\|^2 \right] \\
			                                               & \leq 2\E_t \|\nabla F_{i_t}(x_t, \xi_t^{i_t}) - \nabla f(x_t)\|^2 + 2\E_t\|\nabla f(x_t)\|^2   \\
                                                           & \leq 8\Gamma(2 \theta + 1) \sigma^2 + 4 \zeta^2 + 2\|\nabla f(x_t)\|^2.
		\end{aligned}
	\]
    Thus we arrive at \eqref{eq:squared-clip} by applying \eqref{eq:markov-conc}.
\end{proof}

\subsection{Convergence in expectation}

We begin with the main descent lemma we will use in the analysis of convergence
in expectation.

\begin{lem}
	If $f$ is $L$-smooth and $\alpha_t = \min \left\{ 1, \frac{c}{\|\nabla
			f(x_t)\|} \right\}$, then
	\begin{equation}\label{eq:main-ineq}
        \E_t f(\tilde x_{t+1}) - f(\tilde x_t) \leq - \frac{\eta \alpha_t}{2} \|\nabla f(x_t)\|^2 + \frac{\eta}{2 \alpha_t} \|\E_t g_t^{i_t} - \clip_c(\nabla f(x_t))\|^2 + \frac{\eta^2 L}{2} (\eta c^2 \tau_C^2 L + \E_t \| g_t^{i_t} \|^2).
	\end{equation}
\end{lem}

\begin{proof}
	Since $f$ is $L$-smooth we have that
	\begin{equation}\label{eq:smooth}
		\begin{aligned}
			f(\tilde x_{t+1}) & \leq f(\tilde x_t) + \braket{\tilde x_{t+1} - \tilde x_t, \nabla f(\tilde x_t)} + \frac{L}{2} \|\tilde x_{t+1} - \tilde x_t\|^2 \\
			                  & = f(\tilde x_t) - \eta \braket{g_t^{i_t}, \nabla f(\tilde x_t)} + \frac{\eta^2 L}{2} \|g_t^{i_t}\|^2.
		\end{aligned}
	\end{equation}
	due to Lemma~\ref{lem:smooth}.
	Taking the conditional expectation we have that
	\[
		\E_t f(\tilde x_{t+1}) - f(\tilde x_t) \leq - \eta \braket{\E_t g_t^{i_t}, \nabla f(\tilde x_t)} + \frac{\eta^2 L}{2} \E_t \|g_t^{i_t}\|^2.
	\]
	Turning our attention to the inner product in the right-hand side, we apply
	Young's inequality and $L$-smoothness,
	\[
		\begin{aligned}
			-\eta \braket{\nabla f(\tilde x_t), \E_t g_t} & = - \eta \braket{\nabla f(x_t), \E_t g_t^{i_t}} + \eta \braket{\nabla f(x_t) - \nabla f(\tilde x_t), \E_t g_t^{i_t}}                                                                                                                               \\
			                                              & \leq -\eta \braket{\nabla f(x_t), \E_t g_t^{i_t}} + \frac{\eta}{2}\|\nabla f(x_t) - \nabla f(\tilde x_t)\|^2 + \frac{\eta}{2} \|\E_t g_t^{i_t}\|^2                                                                                                 \\
			                                              & \leq -\eta\braket{\nabla f(x_t), \E_t g_t^{i_t}} + \frac{\eta L^2}{2} \|x_t - \tilde x_t\|^2 + \frac{\eta}{2} \|\E_t g_t^{i_t}\|                                                                                                                   \\
			                                              & = -\frac{\eta}{\alpha_t} \braket{\clip_c(\nabla f(x_t)), \E_t g_t^{i_t}} + \frac{\eta L^2}{2} \|x_t - \tilde x_t\|^2 + \frac{\eta}{2} \|\E_t g_t^{i_t}\|                                                                                                \\
			                                              & = - \frac{\eta \alpha_t}{2} \|\nabla f(x_t)\|^2 - \frac{\eta}{2 \alpha_t} \|\E_t g_t^{i_t}\|^2 + \frac{\eta}{2 \alpha_t} \|\E_t g_t^{i_t} - \clip_c(\nabla f(x_t))\|^2 + \frac{\eta L^2}{2} \|x_t - \tilde x_t\|^2 + \frac{\eta}{2} \|\E_t g_t^{i_t}\|^2 \\
			                                              & \leq - \frac{\eta \alpha_t}{2} \|\nabla f(x_t)\|^2 + \frac{\eta}{2 \alpha_t} \|\E_t g_t^{i_t} - \clip_c(\nabla f(x_t))\|^2 + \frac{\eta L^2}{2} \|x_t - \tilde x_t\|^2, \end{aligned}
	\]
    where the last inequality is due to $\alpha_t \leq 1$, so the two terms
    involving $\|\E_t g_t^{i_t}\|^2$ can be dropped. Applying
    Lemma~\ref{lem:virtual-ineq} we obtain the stated inequality.
\end{proof}

\begin{thm}\label{thm:expectation}
	Suppose Assumptions~\ref{asmp:smooth} and \ref{asmp:sub-weibull} hold. Then
	there exists a constant step size $\eta$ and clipping radius $c$ such that
	for $\epsilon \in (0,1)$, we have $\frac{1}{T} \sum_{t=0}^{T-1}
		\E\|\nabla f(x_t)\| \leq \epsilon$ within
	\begin{equation}
        \widetilde{O} \left( \frac{(\sigma^2 + \zeta^2) L \Delta}{\epsilon^4} + \frac{(\sigma + \zeta) \tau_C L \Delta}{\epsilon^{3}} + \frac{\tau_C L \Delta}{\epsilon^2} \right)
	\end{equation}
	iterations of Algorithm~\ref{algo:asgd_hom}.
\end{thm}

\begin{proof}
	We analyze the convergence in two cases seperately.

    \paragraph{Case I.} We first turn to case when $\|\nabla f(x_t)\| \leq
    c/2$. Applying \eqref{eq:clip-error} and \eqref{eq:squared-clip} from
    Lemma~\ref{lem:clip-error} to our main descent inequality, we have
	\[
		\begin{aligned}
            \E_t f&(\tilde x_{t+1}) - f(\tilde x_t) \\
                  & \leq - \frac{\eta}{2} \|\nabla f(x_t)\|^2 + \eta (4 \sigma^2 \Gamma(2 \theta + 1) + 2 \zeta^2) \exp \left( -\left( \frac{c - 2 \zeta}{4 \sigma} \right)^\frac{1}{\theta} \right) + \frac{\eta^2 L}{2} \left( \eta c^2 \tau_C^2 L + 24 \sigma^2 \Gamma(2 \theta + 1) + 8 \zeta^2 + 2\|\nabla f(x_t)\|^2 \right)                                                                                                                        \\
                  & = - \frac{\eta}{2} \left( 1 - 2 \eta L \right) \|\nabla f(x_t)\|^2 + \eta (4\sigma^2 \Gamma(2 \theta + 1) + 2 \zeta^2) \exp \left( -\left( \frac{c - 2 \zeta}{4 \sigma} \right)^\frac{1}{\theta} \right) + \eta^2 (12 \sigma^2 \Gamma(2 \theta + 1) + 4 \zeta^2) L + \frac{\eta^3 c^2 \tau_C^2 L^2}{2} \\
                  & \leq - \frac{\eta}{4} \|\nabla f(x_t)\|^2 + \eta (4\sigma^2 \Gamma(2 \theta + 1) + 2 \zeta^2) \exp \left( -\left( \frac{c - 2 \zeta}{4 \sigma} \right)^\frac{1}{\theta} \right) + \eta^2 (12 \sigma^2 \Gamma(2 \theta + 1) + 4 \zeta^2) L + \frac{\eta^3 c^2 \tau_C^2 L^2}{2}.
		\end{aligned}
	\]
	where in the last inequality we used $\eta \leq 1 / (4 L)$. Taking the
	unconditional expectation, we have that
	\[
		\frac{1}{4} \E \|\nabla f(x_t)\|^2 \leq \frac{\E [f(\tilde x_t) - f(\tilde x_{t+1})]}{\eta} + (4\sigma^2 \Gamma(2 \theta + 1) + 2 \zeta^2) \exp \left( -\left( \frac{c - 2 \zeta}{4 \sigma} \right)^\frac{1}{\theta} \right) + \eta (12 \sigma^2 \Gamma(2 \theta + 1) + 4 \zeta^2) L + \frac{\eta^2 c^2 \tau_C^2 L^2}{2}.
	\]
    Summing over $t \in \mathcal{T} = \{t\colon \|\nabla f(x_t)\| \leq c/2\}$,
    we obtain
	\[
        \begin{aligned}
            \frac{1}{4 T} \sum_{t \in\mathcal{T}} \E &\|\nabla f(x_t)\|^2 \leq \\
                                                       &\frac{1}{T} \sum_{t \in\mathcal{T}} \left( \frac{\E[f(\tilde x_t) - f(\tilde x_{t+1})]}{\eta}+ (4\sigma^2 \Gamma(2 \theta + 1) + 2 \zeta^2) \exp \left( -\left( \frac{c - 2 \zeta}{4 \sigma} \right)^\frac{1}{\theta} \right) + \eta (12 \sigma^2 \Gamma(2 \theta + 1) + 4 \zeta^2) L + \frac{\eta^2 c^2 \tau_C^2 L^2}{2}\right) .
        \end{aligned}
	\]
	% This in turn implies
	% \[
	% 	\frac{1}{T} \sum_{t \colon \|\nabla f(x_t)\| < \frac{c}{2}} \E\|\nabla f(x_t)\| = O \left( \sqrt{\frac{\Delta}{\eta T}} + \sqrt{\sigma^2 + \zeta^2} \exp \left( -\frac{1}{2}\left( \frac{c - 2 \zeta}{4 \sigma} \right)^\frac{1}{\theta} \right) + \sqrt{\eta (\sigma^2 + \zeta^2) L} + \eta c \tau_C L \right)
	% \]
    \paragraph{Case II.} We now turn to the case when $\|\nabla f(x_t)\| >
    c/2$. Using Jensen's inequality and nonexpansiveness of the clipping
    operator, we have
	\[
            \|\E_t g_t^{i_t} - \clip_c(\nabla f(x_t))\|^2 \leq \E_t \| \nabla F_{i_t}(x_t, \xi_t^{i_t}) - \nabla f(x_t)\|^2 \leq 4 \sigma^2 \Gamma(2 \theta + 1) + 2 \zeta^2.
	\]
	Thus, our main inequality \eqref{eq:main-ineq} in this case implies that
	\[
		\begin{aligned}
			\E [f(\tilde x_{t+1}) - f(\tilde x_t)] & \leq - \frac{\eta \alpha_t}{2} \E \|\nabla f(x_t)\|^2 + \frac{\eta}{2\alpha_t} (4 \sigma^2 \Gamma(2 \theta + 1) + 2\zeta^2)  + \frac{\eta^2 L}{2} (\eta c^2 \tau_C^2 + \E \| g_t \|^2) \\
			                                       & \leq - \frac{\eta \alpha_t}{2} \E \|\nabla f(x_t)\|^2 + \frac{\eta}{c} (4\sigma^2 \Gamma(2 \theta + 1) + 2\zeta^2) \|\nabla f(x_t)\|  + \frac{\eta^2 c^2 L}{2} (1 + \eta \tau_C^2 L)   \\
		\end{aligned}
	\]
    where we used that $\|g_t\| \leq c$ always. Whenever $\|\nabla f(x_t)\| >
    c$, we have $\alpha_t = c / \|\nabla f(x_t)\|$, so if $c \geq 2
    \sqrt{8\Gamma(2 \theta + 1) \sigma^2 + 4 \zeta^2}$ the inequality takes the form
	\[
		\begin{aligned}
			\E [f(\tilde x_{t+1}) - f(\tilde x_t)] & \leq - \frac{\eta c}{2} \left( 1 - \frac{8\sigma^2 \Gamma(2 \theta + 1) + 4 \zeta^2}{c^2}\right) \E \|\nabla f(x_t)\| + \frac{\eta^2 c^2 L}{2} (1 + \eta \tau_C^2 L) \\
			                                       & \leq - \frac{3\eta c}{8} \E \|\nabla f(x_t)\| + \frac{\eta^2 c^2 L}{2} (1 + \eta \tau_C^2 L)                                                             \\
		\end{aligned}
	\]
	Using $\eta \leq 1 / (4\tau_C L)$, we arrive at
	\[
		\frac{3c}{8} \E \|\nabla f(x_t)\| \leq \frac{\E [f(\tilde x_{t+1}) - f(\tilde x_t)]}{\eta} + \frac{\eta c^2 L}{2} ( 1 + \eta \tau_C^2 L).
	\]
	Now consider the case when $c/2 < \|\nabla f(x_t)\| < c$, so $\alpha_t = 1$. In
    this case, we have that $-\|\nabla f(x_t)\| \leq -c/2$ and $1 \leq \|\nabla
    f(x_t)\| / c$. Thus, similarly as in the previous case, our main inequality
    \eqref{eq:main-ineq} implies that
	\[
		\begin{aligned}
			\E_t f(\tilde x_{t+1}) - f(\tilde x_t) & \leq - \frac{\eta c}{4} \|\nabla f(x_t)\| + \frac{\eta}{c} 2 \Gamma(2 \theta + 1) \sigma^2 \|\nabla f(x_t)\| + \frac{\eta^2 c^2 L}{2} (1 + \eta \tau_C^2 L) \\
			                                       & \leq -\frac{\eta c}{4} \left( 1 - \frac{8\sigma^2 \Gamma(2 \theta + 1)}{c^2} \right) \|\nabla f(x_t)\| + \frac{\eta^2 c^2 L}{2} (1 + \eta \tau_C^2 L)       \\
			                                       & \leq -\frac{\eta c}{8} \|\nabla f(x_t)\| + \frac{\eta^2 c^2 L}{2} (1 + \eta \tau_C^2 L).
		\end{aligned}
	\]
	Taking the unconditional expectation, we arrive at the inequality
	\[
        \frac{c}{8} \E \|\nabla f(x_t)\| \leq \frac{\E [f(\tilde x_{t}) - f(\tilde x_{t+1})]}{\eta} + \frac{\eta c^2 L}{2} (1 + \eta \tau_C^2 L).
	\]
	Thus we have that
	\[
        \frac{1}{8T}\sum_{t \in \mathcal{T}^c} c\E \|\nabla f(x_t)\| = \frac{1}{T}\sum_{t \in \mathcal{T}^c} \left( \frac{\E [f(\tilde x_{t}) - f(\tilde x_{t+1})]}{\eta} + \frac{\eta c^2 L}{2} (1 + \eta \tau_C^2 L)\right)  
	\]

    \paragraph{Choosing the parameters.} Putting together the two cases, we
    have that
	\[
        \begin{aligned}
            \frac{1}{8T} \left( \sum_{t\in\mathcal{T}} \E\|\nabla f(x_t)\|^2 + \sum_{t\in\mathcal{T}^c} c \E\|\nabla f(x_t)\| \right) &\leq \frac{1}{T}\sum_{t=0}^{T-1} \frac{\E[f(\tilde x_t) - f(\tilde x_{t+1})]}{\eta}+ (4\sigma^2 \Gamma(2 \theta + 1) + 2 \zeta^2) \exp \left( -\left( \frac{c - 2 \zeta}{4 \sigma} \right)^\frac{1}{\theta} \right) \\
                                                                                                                                                  &\quad + \eta (12 \sigma^2 \Gamma(2 \theta + 1) + 4 \zeta^2) L + \frac{\eta^2 c^2 \tau_C^2 L^2}{2} + \frac{\eta c^2 L}{2} \\
                                &= O \left( \frac{\Delta}{\eta T} + (\sigma^2 + \zeta^2) \left( \eta L + \exp \left( -\left( \frac{c - 2 \zeta}{4 \sigma} \right)^\frac{1}{\theta} \right) \right) + \eta^2 c^2 \tau_C^2 L^2 + \eta c^2 L \right).
        \end{aligned}
	\]
    This means that both 
    \[
        \frac{1}{T} \sum_{t\in\mathcal{T}} \E\|\nabla f(x_t)\| = O \left( \sqrt{\frac{\Delta}{\eta T}} + \sqrt{\eta (\sigma^2 + \zeta^2) L} + \sqrt{\sigma^2 + \zeta^2} \exp \left( -\left( \frac{c - 2 \zeta}{4 \sigma} \right)^\frac{1}{\theta} \right) + \eta c \tau_C L + \sqrt{\eta c^2 L} \right)
    \]
    and 
    \[
        \frac{1}{T} \sum_{t\in\mathcal{T}^c} \E\|\nabla f(x_t)\| = O \left( \frac{\Delta}{\eta c T} + \frac{\sigma^2 + \zeta^2}{c} \left( \eta L + \exp \left( -\left( \frac{c - 2 \zeta}{4 \sigma} \right)^\frac{1}{\theta} \right) \right) + \eta^2 c \tau_C^2 L^2 + \eta c L \right).
    \]

    We choose the clipping radius to be $c = 2 \sqrt{8\Gamma(2 \theta + 1)
    \sigma^2} \log^\theta T + 4 \zeta$. Assuming $T$ is large enough so that $c
    \geq 1$, we have that
	\[
        \frac{1}{T} \sum_{t=0}^{T-1} \E\|\nabla f(x_t)\| = \widetilde{O} \left( \sqrt{\frac{\Delta}{\eta T}} + \sqrt{\eta (\sigma^2 + \zeta^2) L} + \eta (\sigma + \zeta) \tau_C L \log^\theta T + \sqrt{\frac{\sigma^2 + \zeta^2}{T}} \right).
	\]
	Then choosing the step size to be
	\[
        \eta = \min \left\{ \frac{1}{4 \tau_C L}, \left( \frac{\Delta}{(\sigma^2 + \zeta^2) LT} \right)^{1/2}, \left( \frac{\Delta}{(\sigma^2 + \zeta^2) \tau_C^2 L^2 T} \right)^{1/3} \right\},
	\]
	we obtain
	\[
        \frac{1}{T} \sum_{t=0}^{T-1} \E\|\nabla f(x_t)\| = \widetilde{O} \left( \sqrt{\frac{\tau_C L \Delta}{T}} + \left( \frac{(\sigma^2 + \zeta^2) L \Delta}{T} \right)^{1/4} + \left( \frac{(\sigma + \zeta) \tau_C L \Delta}{T} \right)^{1/3} + \sqrt{\frac{\sigma^2 + \zeta^2}{T}} \right).
	\]
	Hence, we require
	\[
        \widetilde{O} \left( \frac{(\sigma^2 + \zeta^2) L \Delta}{\epsilon^4} + \frac{(\sigma + \zeta) \tau_C L \Delta}{\epsilon^3} + \frac{\tau_C L \Delta}{\epsilon^2}\right).
	\]
    iteration to reach $\epsilon$-stationarity.
	% Then, noting that in $\tau_n$ seconds, worker $i$ can compute $\lfloor
	% 	\tau_n/\tau_i \rfloor$ gradients, so in the same time
	% Algorithm~\ref{algo:asgd_hom} computes $\sum_{i=1}^{n} \lfloor
	% 	\tau_i/\tau_n \rfloor$ gradients. Thus Algorithm~\ref{algo:asgd_hom}
	% computes a gradient at most every $(1/\tau_n)\sum_{i=1}^{n} \lfloor
	% 	\tau_i/\tau_n \rfloor \leq \sum_{i=1}^{n} (1/\tau_i)$ seconds on
\end{proof}

\subsection{Convergence with high probability}

\begin{lem}\label{lem:descent}
	If $f$ is $L$-smooth and $\alpha_t = \min \{1, c / \|\nabla f(x_t)\|\}$, then
    \begin{equation}\label{eq:high-prob-descent}
		\begin{aligned}
            f(\tilde x_{t+1}) - f(\tilde x_t) & \leq -\frac{\eta}{2} \|\nabla f(x_t)\|^2 + \frac{\eta}{2} \|\E_t g_t^{i_t} - \nabla f(x_t)\|^2 + \frac{\eta^3 c^2 \tau_C^2 L^2}{2} + \frac{\eta^2 L}{2} \|g_t^{i_t}\|^2 \\
			                                  & \quad - \eta \braket{\nabla f(x_t), g_t^{i_t} - \E_t g_t^{i_t}} - \eta \braket{\nabla f(\tilde x_t) - \nabla f(x_t), g_t^{i_t} - \E_t g_t^{i_t}}.
		\end{aligned}
    \end{equation}
\end{lem}

\begin{proof}
	Due to $L$-smoothness we have that
    \[
		\begin{aligned}
			f(\tilde x_{t+1}) & \leq f(\tilde x_t) + \braket{\tilde x_{t+1} - \tilde x_t, \nabla f(\tilde x_t)} + \frac{L}{2} \|\tilde x_{t+1} - \tilde x_t\|^2 \\
			                  & = f(\tilde x_t) - \eta \braket{g_t^{i_t}, \nabla f(\tilde x_t)} + \frac{\eta^2 L}{2} \|g_t^{i_t}\|^2.
		\end{aligned}
    \]
	due to Lemma~\ref{lem:smooth}. We first expand the inner product according to
    \begin{align}
        - \eta \braket{\nabla f(\tilde x_t), g_t^{i_t}}  = &-\eta \braket{\nabla f(x_t), \E_t g_t^{i_t}} \tag*{\text{(deterministic descent)}} \\
                                                           & - \eta \braket{\nabla f(\tilde x_t) - \nabla f(x_t), \E_t g_t^{i_t}}  \tag*{(delay error)} \\
                                                           & \underbrace{- \eta \braket{\nabla f(x_t), g_t^{i_t} - \E_t g_t^{i_t}} - \eta \braket{\nabla f(\tilde x_t) - \nabla f(x_t), g_t^{i_t} - \E_t g_t^{i_t}}}_{Z_t} \tag*{(martingale difference)},
    \end{align}
    where we note that the last two terms form a martingale difference sequence
    $Z_t$ adapted to $\mathcal{G}_t = \mathcal{F}_{t+1}$. The first inner
    product in the right-hand side can be written as 
    \begin{equation}\label{eq:1st-inner-prod}
        -\eta \braket{\nabla f(x_t), \E_t g_t^{i_t}} = -\frac{\eta}{2} \|\nabla f(x_t)\|^2 - \frac{\eta}{2} \|\E_t g_t^{i_t}\|^2 + \frac{\eta}{2} \|\E_t g_t^{i_t} - \nabla f(x_t)\|^2.
    \end{equation}
    Using Young's inequality (Lemma~\ref{lem:young}) on the second inner
    product gives the bound
    \begin{equation}\label{eq:2nd-inner-prod}
		- \eta \braket{\nabla f(\tilde x_t) - \nabla f(x_t), \E_t g_t^{i_t}} \leq \frac{\eta}{2}\|\nabla f(\tilde x_t) - \nabla f(x_t)\|^2 + \frac{\eta}{2} \|\E_t g_t^{i_t}\|^2.
    \end{equation}
    Moreover, due to $L$-smoothness of $f$ (Assumption~\ref{asmp:smooth}) and
    Lemma~\ref{lem:virtual-ineq}, we have
	\[
		\frac{\eta}{2}\|\nabla f(\tilde x_t) - \nabla f(x_t)\|^2 \leq \frac{\eta L^2}{2} \|\tilde x_t - x_t\|^2 \leq \frac{\eta^3 c^2 L^2}{2}
	\]
    By subtracting $f(\tilde x_t)$ from both sides of \eqref{eq:smooth}, we
    arrive at the stated inequality by noting that the
    $\frac{\eta}{2}\E_t\|g_t\|^2$ terms in \eqref{eq:1st-inner-prod} and
    \eqref{eq:2nd-inner-prod} cancel.
\end{proof}

\begin{thm}\label{thm:sub-weibull-hom-high-prob}
	Suppose Assumptions~\ref{asmp:smooth} and \ref{asmp:sub-weibull} hold. Then
	there exists a constant step size $\eta$ and clipping radius $c$ such that
	for $\epsilon \in (0,1)$ and failure probability $\delta \in
		(0,1)$, we have $\prob \left( \frac{1}{T} \sum_{t=0}^{T-1} \|\nabla
		f(x_t)\| \leq \epsilon \right) \geq 1 - \delta$ within
	\begin{equation}
        \widetilde{O} \left( \frac{\sigma^2 \log^{2 \theta}(1/\delta)}{\epsilon^4} + \frac{\sigma \tau_C \log^{\theta}(1/\delta)}{\epsilon^{3}} + \frac{\tau_C}{\epsilon^2} \right)
	\end{equation}
	iterations of Algorithm~\ref{algo:asgd_hom}.
\end{thm}

\begin{proof}
	Similarly as before, we analyze the convergence in two different cases seperately.

    \paragraph{Case I.} We begin by treating the case when $\|\nabla f(x_t)\|
    \leq c/2$, in which we will show that the martingale difference sequence
    concentrates. First, we denote the martingale difference term 
    \[
        Z_t = - \eta \braket{\nabla f(x_t), g_t^{i_t} - \E_t g_t^{i_t}}  - \eta \braket{\nabla f(\tilde x_t) - \nabla f(x_t), g_t^{i_t} - \E_t g_t^{i_t}}
    \]
    and the indicator function $\psi_t = \{\|\nabla f(x_t)\| \leq c / 2\}$ %set $\mathcal{T} = \{t \colon \|\nabla f(x_t)\| \leq c / 2\}$.
    Since clipping ensures $\|g_t^{i_t}\| \leq c$ almost surely, we can bound
    $|Z_t|$ uniformly for $t \in \mathcal{T}$:
	\[
        \begin{aligned}
            |\psi_t Z_t| &\leq \eta (\|\nabla f(x_t)\| + \|\nabla f(\tilde x_t) - \nabla f(x_t)\|) \left( \|g_t^{i_t}\| + \|\E_t g_t^{i_t}\|\right)  \\
                  &\leq 2\eta \left( \frac{c}{2} + \eta c \tau_C L \right) c \\
                  &\leq \frac{3 \eta c^2}{4}
        \end{aligned}
	\]
    where in the second inequality we use that $\eta \leq 1/(4\tau_C L)$. By
    applying the Cauchy-Schwarz inequality  and Young's inequality we can also
    bound the conditional second moment of $Z_t$:
	\[
		\begin{aligned}
            \E [Z_t^2 \mid \mathcal{F}_t] & \leq \eta^2 \|\nabla f(x_t) + \nabla f(\tilde x_t) - \nabla f(x_t)\|^2 \|g_t^{i_t} - \E_t g_t^{i_t}\|^2                                               \\
			                          & \leq \eta^2 \left( 2\|\nabla f(x_t)\|^2 + 2\|\nabla f(\tilde x_t) - \nabla f(x_t)\|^2 \right) \left( 2\|g_t^{i_t}\|^2 + 2\|\E_t g_t^{i_t}\|^2 \right) \\
			                          & \leq 4 \eta^2 c^2 \left( \|\nabla f(x_t)\|^2 + \eta^2 c^2 \tau_C^2 L^2 \right).
		\end{aligned}
	\]
    So applying Freedman's inequality (Lemma~\ref{lem:freedman-ineq}) with
    $\rho = 3/16$, we have with probability at least $1 - \delta / 2$ that
	\[
		\begin{aligned}
			\sum_{t \in \mathcal{T}} Z_t & \leq \frac{4\rho \eta^2 c^2}{3 \eta c^2} \sum_{t \in \mathcal{T}} \left( \|\nabla f(x_t)\|^2 + \eta^2 c^2 \tau_C^2 L^2 \right) + \frac{3 \eta c^2 }{4\rho} \log \frac{2}{\delta} \\
			                             & \leq \frac{1}{4} \sum_{t \in \mathcal{T}} \left( \eta \|\nabla f(x_t)\|^2 + \eta^3 c^2 \tau_C^2 L^2 \right) + 4 \eta c^2 \log \frac{2}{\delta}.
		\end{aligned}
	\]
    Thus, summing \eqref{eq:high-prob-descent} over $t \in \mathcal{T}$, we
    have
    \[
        \begin{aligned}
            \sum_{t\in\mathcal{T}} (f(\tilde x_{t+1}) - f(\tilde x_t)) &\leq \sum_{t\in\mathcal{T}} \left( -\frac{\eta}{2} \|\nabla f(x_t)\|^2 + \frac{\eta}{2} \|\E_t g_t^{i_t} - \nabla f(x_t)\|^2 + \frac{\eta^3 c^2 \tau_C^2 L^2}{2} + \frac{\eta^2 L}{2} \|g_t^{i_t}\|^2 + Z_t \right) \\
                                                                       &\leq \sum_{t\in\mathcal{T}} \left( -\frac{\eta}{4} \|\nabla f(x_t)\|^2 + \frac{\eta}{2} \|\E_t g_t^{i_t} - \nabla f(x_t)\|^2 + \frac{3\eta^3 c^2 \tau_C^2 L^2}{4} + \frac{\eta^2 c^2 L}{2} \right) + 4 \eta c^2 \log \frac{2}{\delta}
        \end{aligned}
    \]
    under the same probability. Applying \eqref{eq:clip-error} from
    Lemma~\eqref{lem:clip-error} to the clipping error term $\|\E_t g_t^{i_t} -
    \nabla f(x_t)\|^2$ and dividing by $\eta$, we have
	\[
		\begin{aligned}
            \frac{1}{4}\sum_{t\in\mathcal{T}} & \|\nabla f(x_t)\|^2 \\
                                   & \leq \sum_{t \in \mathcal{T}} \left( \frac{f(\tilde x_{t}) - f(\tilde x_{t+1})}{\eta} + (4 \sigma^2 \Gamma(2 \theta + 1) + 2 \zeta^2) \exp \left( - \left( \frac{c - 2 \zeta}{4 \sigma} \right)^\frac{1}{\theta} \right) + \frac{3 \eta^2 c^2 \tau_C^2 L^2}{4} + \frac{\eta c^2 L}{2} \right) + 4 c^2 \log \frac{2}{\delta}
			            % & \leq O \left( \frac{\Delta}{\eta T} + (8 \sigma^2 \Gamma(2 \theta + 1) + 4 \zeta^2) \exp \left( - \left( \frac{c - 2 \zeta}{4 \sigma} \right)^\frac{1}{\theta} \right) +  \eta^2 c^2 \tau_C^2 L^2 + \frac{\eta c^2 L}{2} + \frac{c^2 \log (1/\delta)}{T} \right)
		\end{aligned}
	\]
	with probability at least $1 - \delta / 2$.

    \paragraph{Case II.} We consider the case $\|\nabla f(x_t)\| > c / 2$. By
    smoothness, we have
	\[
		\begin{aligned}
			f(\tilde x_{t+1}) - f(\tilde x_t) & \leq - \frac{\eta \alpha_t}{2} \|\nabla f(x_t)\|^2 + \frac{\eta}{2 \alpha_t} \|g_t^{i_t} - \clip_c(\nabla f(x_t))\|^2 + \frac{\eta^3 c^2 \tau_C^2 L^2}{2} + \frac{\eta^2 c^2 L}{2}          \\
                                              & \leq - \frac{\eta \alpha_t}{2} \|\nabla f(x_t)\|^2 + \frac{\eta}{2 \alpha_t} \|\nabla F_{i_t}(x_t, \xi_t^{i_t}) - \nabla f(x_t)\|^2 + \frac{\eta^3 c^2 \tau_C^2 L^2}{2} + \frac{\eta^2 c^2 L}{2}.
		\end{aligned}
	\]
	where the second inequality holds due to nonexpansiveness of the clipping
	operator. Then, under sub-Weibull noise (Assumption~\ref{asmp:sub-weibull}),
	\[
        \begin{aligned}
            \sum_{t \in \mathcal{T}^c} \|\nabla F_{i_t}(x_t, \xi_t^{i_t}) - \nabla f(x_t)\|^2 & \leq \sum_{t \in \mathcal{T}^c} (2\|\nabla F_{i_t}(x_t, \xi_t^{i_t}) - \nabla f_{i_t}(x_t)\|^2 + 2\|\nabla f_{i_t}(x_t) - \nabla f(x_t)\|^2) \\
                                                                   & \leq 2\sigma^2 \log^{2 \theta} (2 T / \delta) + 2 \zeta^2
        \end{aligned}
	\]
    with probability at least $1 - \delta / 2$. Thus for all $t \in
    \mathcal{T}^c$, we have under the same probability that if $\|\nabla
    f(x_t)\| \leq c$
	\[
		\begin{aligned}
            f(\tilde x_{t+1}) - f(\tilde x_t) & \leq - \frac{\eta \alpha_t}{2} \|\nabla f(x_t)\|^2 + \frac{\eta}{\alpha_t} (\sigma^2 \log^{2 \theta} (4 T / \delta) + \zeta^2) + \frac{\eta^3 c^2 \tau_C^2 L^2}{2} + \frac{\eta^2 c^2 L}{2} \\
			                                  & \leq - \frac{\eta c}{4} \|\nabla f(x_t)\| + 2 \frac{\eta}{c} (\sigma^2 \log^{2 \theta} (4 T / \delta) + \zeta^2) \|\nabla f(x_t)\| + \frac{\eta^3 c^2 \tau_C^2 L^2}{2} + \frac{\eta^2 c^2 L}{2}                   \\
                                              & \leq - \frac{\eta c}{4} \left( 1 - \frac{8(\sigma^2 \log^{2 \theta} (4 T / \delta) + \zeta^2)}{c^2} \right) \|\nabla f(x_t)\| + \frac{\eta^3 c^2 \tau_C^2 L^2}{2} + \frac{\eta^2 c^2 L}{2},
		\end{aligned}
	\]
    whereas if $\|\nabla f(x_t)\| \geq c$,
	\[
		\begin{aligned}
			f(\tilde x_{t+1}) - f(\tilde x_t) & \leq - \frac{\eta \alpha_t}{2} \|\nabla f(x_t)\|^2 + \frac{\eta}{\alpha_t} (\sigma^2 \log^{2 \theta} (2 / \delta) + \zeta^2) + \frac{\eta^3 c^2 \tau_C^2 L^2}{2} + \frac{\eta^2 c^2 L}{2} \\
			                                  & \leq - \frac{\eta c}{2} \|\nabla f(x_t)\| + \frac{\eta}{c} (\sigma^2 \log^{2 \theta} (2 / \delta) + \zeta^2) \|\nabla f(x_t)\| + \frac{\eta^3 c^2 \tau_C^2 L^2}{2} + \frac{\eta^2 c^2 L}{2}                   \\
			                                  & \leq - \frac{\eta c}{2} \left( 1 - \frac{2(\sigma^2 \log^{2 \theta} (2 / \delta) + \zeta^2)}{c^2} \right)  \|\nabla f(x_t)\| + \frac{\eta^3 c^2 \tau_C^2 L^2}{2} + \frac{\eta^2 c^2 L}{2}.
		\end{aligned}
	\]
    Hence, choosing $c \geq 4 \sqrt{\sigma^2 \log^{2\theta} (2/\delta) +
    \zeta^2}$, we have
	\[
        \frac{1}{8T} \sum_{t \in \mathcal{T}^c} c \|\nabla f(x_t)\| \leq \sum_{t \in \mathcal{T}^c} \left( \frac{f(\tilde x_{t+1}) - f(\tilde x_{t})}{\eta} + \frac{\eta^2 c^2 \tau_C^2 L^2}{2} + \frac{\eta c^2 L}{2} \right) 
	\]
	with probability at least $1 - \delta / 2$. 

    \paragraph{Choosing the parameters.} Putting the two cases together, we
    have that
    \[
        \begin{aligned}
            \frac{1}{8T} \left( \sum_{t\in\mathcal{T}} \|\nabla f(x_t)\|^2 + \sum_{t\in\mathcal{T}^c} c\|\nabla f(x_t)\|\right) &\leq \frac{1}{T} \sum_{t=0}^{T-1} \frac{f(\tilde x_{t}) - f(\tilde x_{t+1})}{\eta} + \frac{3\eta^2 c^2 \tau_C^2 L^2}{4} + \frac{\eta c^2 L}{2} \\
                                   &\quad + (4 \sigma^2 \Gamma(2 \theta + 1) + 2 \zeta^2) \exp \left( - \left( \frac{c - 2 \zeta}{4 \sigma} \right)^\frac{1}{\theta} \right) + 4 c^2 \log \frac{2}{\delta} 
        \end{aligned}
    \]
    with probability  atleast $1 - \delta$. Using the same technique as in the
    proof of Theorem~\ref{thm:expectation}, we have that 
    \[
        \frac{1}{T} \sum_{t=0}^{T-1} \|\nabla f(x_t)\| = O \left( \frac{\Delta}{\eta c T} + \sqrt{\frac{\Delta}{\eta T}} + \sqrt{\eta c^2 L} + \eta c \tau_C L + \sqrt{\sigma^2 + \zeta^2} \exp \left( -\frac{1}{2} \left( \frac{c - 2 \zeta}{4 \sigma} \right)^\frac{1}{\theta} \right) + \sqrt{\frac{c^2 \log(1 / \delta)}{T}} \right)
	\]
	with probability at least $1 - \delta$. Then if
	\[
		c = \max \left\{ 4 \sqrt{\sigma^2 \log^{2\theta} (2/\delta) + \zeta^2}, 4 \sigma \log^\theta T + 2 \zeta \right\}
	\]
	and assuming $T$ is large enough that $c \geq 1$, we have
	\[
		\frac{1}{T} \sum_{t=0}^{T-1} \|\nabla f(x_t)\| = \widetilde{O} \left( \sqrt{\frac{\Delta}{\eta T}} + \left( \sqrt{\eta (\sigma^2 + \zeta^2) L} + \eta (\sigma + \zeta) \tau_C L \right) \log^{\theta}(T/\delta) + \sqrt{\frac{\sigma^2 \log^{2 \theta + 1}(1 / \delta) + \zeta^2}{T}} \right)
	\]
	with probability at least $1 - \delta$. Then choosing
	\[
		\eta = \min \left\{ \frac{1}{4 \tau_C L}, \left( \frac{\Delta}{(\sigma^2 + \zeta^2) \log^{2 \theta}(T / \delta) L T} \right)^{1/2}, \left( \frac{\Delta}{(\sigma^2 + \zeta^2) \tau_C^2 L^2 T} \right)^{1/3} \right\},
	\]
	we have
	\[
        \begin{aligned}
            \frac{1}{T} &\sum_{t=0}^{T-1} \|\nabla f(x_t)\| \\
                        &= \widetilde{O} \left( \left( \frac{(\sigma^2 + \zeta^2) \log^{2 \theta}(T/\delta) L\Delta}{T} \right)^{1/4} + \left( \frac{(\sigma + \zeta) \tau_C \log^\theta(T/\delta) L \Delta}{T} \right)^{1/3} + \left( \frac{\tau_C L \Delta}{T} \right)^{1/2} + \left(\frac{\sigma^2 \log^{2 \theta + 1}(1 / \delta) + \zeta^2}{T} \right)^{1/2} \right)
        \end{aligned}
	\]
    with the same probability. Then, for $\epsilon \in (0,1)$, it holds that
    $\prob (\frac{1}{T} \sum_{t=0}^{T-1} \|\nabla f(x_t)\| \leq \epsilon) \geq
    1 - \delta$ after
	\[
        \widetilde{O} \left( \frac{(\sigma^2 + \zeta^2) \log^{2 \theta}(1/\delta) L \Delta}{\epsilon^4} + \frac{(\sigma + \zeta) \tau_C \log^\theta (1/\delta) L \Delta}{\epsilon^{3}} + \frac{\tau_C L \Delta}{\epsilon^2} + \frac{\sigma^2\log^{2 \theta + 1} (1/\delta) + \zeta^2}{\epsilon^2} \right)
	\]
	iterations.

\end{proof}
\end{document}